\documentclass[a4paper,11pt,]{article}
\usepackage{amsmath}
\usepackage{amssymb}
\usepackage{theorem}
\usepackage{mathrsfs} 
\usepackage[usenames,dvipsnames]{pstricks}
\usepackage{euscript}
\usepackage{graphicx}
\usepackage{charter}
\usepackage{mdframed}
\usepackage{enumerate, subfigure, color,amssymb}
\usepackage[colorlinks,hyperindex]{hyperref}
\def\endproof{\vbox{\hrule height0.6pt\hbox{\vrule height1.3ex%
width0.6pt\hskip0.8ex\vrule width0.6pt}\hrule height0.6pt}}
\newtheorem{theorem}{Theorem}[section]
\newtheorem{lemma}[theorem]{Lemma}

\newtheorem{proposition}[theorem]{Proposition}
\theoremstyle{plain}{\theorembodyfont{\rmfamily}%
}
\theoremstyle{plain}{\theorembodyfont{\rmfamily}%
}
\theoremstyle{plain}{\theorembodyfont{\rmfamily}%
}
\theoremstyle{plain}{\theorembodyfont{\rmfamily}%
}
\theoremstyle{plain}{\theorembodyfont{\rmfamily}%
}
\theoremstyle{plain}{\theorembodyfont{\rmfamily}%
\newtheorem{remark}[theorem]{Remark}}
\theoremstyle{plain}{\theorembodyfont{\rmfamily}%
}
\theoremstyle{plain}{\theorembodyfont{\rmfamily}%

\newcommand{\email}[1]{\href{mailto:#1}{\nolinkurl{#1}}}
\def\argmin{\operatornamewithlimits{argmin}}
\DeclareMathOperator{\E}{\mathbb{E}}
\renewcommand{\P}{\mathbb{P}}

\newtheorem{ass}{Assumption}

\providecommand{\nor}[1]{\left\lVert {#1} \right\rVert}
\providecommand{\norh}[1]{\Vert {#1} \Vert_\hh}

\providecommand{\scalh}[2]{\langle{#1},{#2}\rangle_\hh}

\providecommand{\scal}[2]{\left\langle{#1},{#2}\right\rangle}

\newcommand{\R}{\mathbb R}

\newcommand{\N}{\mathbb N}

\newcommand{\XX}{\mathcal X}
\newcommand{\hh}{\mathcal H}
\newcommand{\EE}{\mathcal E}
\newcommand{\emp}{\hat{\mathcal E}}

\newcommand{\NN}{\mathbb{N}}
\newcommand{\tr}{\mathbf{z}}

\newcommand{\LK}{L}

\newcommand{\RPP}{\mathbb{R}_{++}}

\title{\sffamily\LARGE Learning with Incremental Iterative Regularization}
 \author{Lorenzo Rosasco$^1$ \; 
and\; Silvia Villa$^2$\\[3mm]
\small
\small $\!^1$ Universit\`a degli Studi di Genova\\
\small DIBRIS, Via Dodecaneso 35 -- 16146 Genova\\
\small\&\\
\small Massachusetts Institute of Technology 
and Istituto Italiano di Tecnologia\\
\small Laboratory for Computational and Statistical Learning,
Cambridge, MA 02139, USA\\
\small \email{lrosasco@mit.edu}
\\[3mm]
\small $\!^2$Massachusetts Institute of Technology 
and Istituto Italiano di Tecnologia\\
\small Laboratory for Computational and Statistical Learning,
Cambridge, MA 02139, USA\\
\small \email{silvia.villa@iit.it}
}

\begin{document}

\maketitle

\begin{abstract}
Within a statistical learning setting,  we propose and study an iterative regularization algorithm for least squares defined by  an incremental gradient method.   In particular, we show that, if all other parameters are fixed a priori, the number of passes over the data (epochs) acts as a regularization 
parameter, and  prove strong universal consistency, i.e.  almost sure convergence of the risk, as well as  sharp finite sample bounds for 
the iterates. Our  results are a step towards understanding the effect of multiple epochs in  stochastic 
gradient techniques in machine learning and rely  on  integrating  statistical and optimization
results.
\end{abstract}


\section{Introduction}
 Machine learning applications often require efficient statistical procedures to  process potentially massive amount of high dimensional data. Motivated by such applications, the  broad objective of our study  is deriving learning procedures with optimal statistical properties, and, at the same time,   computational complexities  proportional to the {\em generalization} properties allowed by the  data, rather than their raw amount \cite{bottou-bousquet-2011}. 
In this paper, we focus on  iterative regularization  as a viable approach towards this goal.  The key observation behind these techniques  is 
that  iterative optimization schemes applied to 
scattered, noisy data
exhibit a self-regularizing property, in the sense that early termination 
(early-stop) of the iterative process has a regularizing effect \cite{Nem86, Pol87}. 
Indeed, iterative regularization algorithms  are classical in inverse problems \cite{Eng96}, and
 have been  recently considered in machine learning \cite{ZhangYu03,yao,bauer,BlaKra10,CapYao06,raskutti},  
 where they have been proved  to achieve optimal learning bounds,  matching those of variational regularization schemes such as  Tikhonov  \cite{Caponnetto:2006,SteinwartHS09}.
 
In this paper, we consider an iterative regularization algorithm for the square loss, based on a recursive procedure    updating the solution after processing one training set point at each iteration. Methods of the  latter form,  often broadly referred to as online learning  algorithms, 
have become standard in the processing of large data-sets, because of their often low iteration cost and good practical performance. Theoretical studies  for this class of algorithms have been developed within different frameworks. In composite optimization
\cite{nedic2001incremental}, in stochastic optimization
\cite{Nem09, SreSriTew12}, in 
 online learning, a.k.a. sequential prediction 
\cite{Cesal},  
and finally, in statistical learning
\cite{Cesa}.  
The latter is the setting of interest in this paper, where we aim at developing an analysis keeping into account simultaneously 
both statistical and computational aspects.  To place our contribution in  context, it is useful to emphasize the role of regularization and different ways in which 
it can be incorporated in online learning algorithms.  The key idea of regularization is that  controlling the {\em complexity}  of a solution 
can help avoiding overfitting and ensure  stability to   achieve improved results \cite{vapnik1998statistical}.
 Classically, regularization is achieved
perturbing (penalizing) the objective function with some suitable functional, or replacing the original risk minimization problem
by a constrained problem obtained restricting  the space of possible solutions \cite{vapnik1998statistical}.  
Model selection is then performed to determine the amount of regularization suitable for the data at hand.
More recently, there has been an interest in alternative,  possibly more efficient, ways to incorporate regularization.  
We mention in particular \cite{yiming,BacDie14} (see also \cite{TarYao14}) 
where there is no explicit regularization by penalization, and the step-size of an iterative procedure is shown to act as a regularization parameter. 
Here, for each fixed step-size, each  data point is  processed once, but multiple passes are indeed needed to perform model selection (that is  pick the best step-size).   We also mention \cite{Pistol} where  an interesting  adaptive approach is  proposed, which seemingly avoid model selection under certain assumptions. 

In this paper, we consider a different regularization strategy, which is widely used in practice. Namely, we consider no explicit penalization, 
fix the step size a priori, and analyze the effect of the number of passes  over the data, which  becomes the only free parameter to avoid overfitting, i.e. regularize. 
The associated regularization strategy, that we dub {\em incremental iterative regularization}, is  hence based on  early stopping. The latter
 is a well known "trick", for example in training large neural networks \cite{lecun-98b}, 
 and is known to  perform very  well in practice \cite{HuaAvr14}.
  Our goal is to provide a theoretical understanding of the generalization property of the above heuristic for 
  incremental/online techniques by grounding it  in solid theoretical terms. 
Towards this end, we develop a theoretical analysis considering  the behavior of both the excess risk 
and the iterates themselves.  For the latter we  obtain sharp finite sample bounds  matching those for Tikhonov regularization in the same setting.   Finite sample bounds for the excess risk can then be easily  derived, albeit possibly suboptimal. Our results are developed in a capacity independent  setting \cite{CucZho07,SteiChri08},  that is under no  conditions on  the covering or entropy numbers \cite{SteiChri08}.
 In this sense our analysis is  worst case and dimension free.  To the best of our knowledge the analysis in the paper is the first theoretical study of regularization by early  stopping in incremental/online algorithms, and thus a first step towards understanding the effect of  multiple epochs of stochastic  gradient  for risk minimization.


The rest of the paper is organized as follows. In Section~\ref{sec:setting} we describe the setting and 
the main  assumptions,  and in Section~\ref{sec:main} we state the main results, discuss them and provide the main elements 
of the proof, which is deferred to the supplementary material. In Section~\ref{sec:exp} we present some experimental results on real and synthetic datasets.
\\
{\bf Notation } We denote by $\mathbb{R}_+=[0,+\infty[\,$,  $\mathbb{R}_{++}=\,]0,+\infty[\,$, 
and $\N^*=\N\smallsetminus\{0\}$. Given a normed space $\mathcal{B}$ and linear operators 
$(A_i)_{1\leq i\leq m}$, $A_i\colon\mathcal{B}\to \mathcal{B}$ for every $i$, their composition 
$A_m\circ\cdots\circ A_1$  will be denoted as $\prod_{i=1}^m A_i$. By convention, if $j>m$, we 
set $\prod_{i=j}^m A_i=I$, where $I$ is the identity of $\mathcal{B}$. The  operator norm will be denoted by 
$\|\cdot\|$ and the Hilbert-Schmidt norm by $\|\cdot\|_{HS}$. Also, if $j>m$, we set $\sum_{i=j}^m A_i=0$. 

\section{Setting and Assumptions}
\label{sec:setting}

We first describe the setting we consider and then introduce and discuss 
the main 
assumptions that will  hold throughout the paper. We essentially follow the framework proposed in \cite{devito05,SmaZho05} 
and further  developed in a series of follow up works \cite{Caponnetto:2006,bauer,SmaZho07,CapYao06}. 
Unlike these papers where a reproducing kernel Hilbert space (RKHS)  setting is considered 
(see  Section~\ref{sec:somspe}), here we develop an equivalent formulation within an abstract Hilbert space. 
This latter formulation is  close to the setting of functional regression \cite{RamSil05} and reduces to standard 
linear regression if $\hh$ is finite dimensional,  see  Section~\ref{sec:somspe}.\\
Let $\hh$ be a separable Hilbert space with inner product  and norm  denoted  by  
$\langle \cdot,\cdot\rangle_\hh$ and $\norh{\cdot}$.  Let $(X,Y)$ be a pair of random variables on a probability space $(\Omega,\mathfrak{S}, \P)$,  with values in $\hh$ and $\R$, respectively. Denote by $\rho$ the distribution of $(X,Y)$, by $\rho_\hh$  the marginal measure on $\hh$, and by  $\rho(\cdot |x)$ the conditional measure on $\R$ given $x\in \hh$. 
Considering the square loss function,  the problem 
under study  is the minimizazion of the {\em risk}, 
\begin{equation}\label{expmin}
\inf_{w\in\hh} \EE(w), \quad \EE(w)=\int_{\hh\times\mathbb{R}} (\scal{w}{x}_\hh-y)^2 d\rho(x,y)\,,
\end{equation}
provided the  distribution $\rho$ is fixed but known only 
through a {\em training set} $\tr=\{(x_1, y_1), \dots, (x_n, y_n)\}$, that is  a
realization of $n\in\N^*$ independent identical copies of $(X,Y)$.
In the following, we  measure the quality of an approximate solution $\hat w\in\hh$ (an estimator) controlling the excess risk 
$$
\EE(\hat w)-\inf_{\hh}\EE.
$$
If the set of solutions of Problem~\eqref{expmin} is non empty, that is $\mathcal{O}=\argmin_{\hh}\EE\neq\varnothing$, 
we also consider
\begin{equation}
\label{eq:wdag}
\nor{\hat w - w^\dagger}_\hh, \quad \text{where} \quad w^\dagger=\argmin_{w\in \mathcal{O}} \norh{w}.
\end{equation}
More precisely we are  interested in deriving almost sure convergence results and finite sample bounds on the above error measures. 
This requires making some assumptions that we discuss next.
We make throughout the following basic assumption.
\begin{ass}\label{ass:zero}
There exist $M\in\left]0,+\infty\right[$ and $\kappa\in\left]0,+\infty\right[$ such 
that $\vert y\vert\leq M$ $\rho$-almost surely,  and  $\norh{x}^2\leq \kappa$ $\rho_\hh$-almost surely. 
\end{ass}
The above assumption is fairly standard. The boundness assumption on the output is 
satisfied in classification, see Section~\ref{sec:somspe}, and can be easily relaxed, see e.g. \cite{Caponnetto:2006}. The boundness 
assumption on the input can also be relaxed, but the resulting  analysis is  more involved. We omit these developments for the sake of clarity. 
It is well known that (see e.g. \cite{devros04}),  under Assumption~\ref{ass:zero}, 
the risk is a convex and continuous functional on  $L^2(\hh,\rho_\hh)$,  the space of square-integrable functions  
with norm $\|f\|^2_\rho=\int_{\hh\times\mathbb{R}}|f(x)|^2d\rho_\hh(x)$. 
The minimizer of the risk on $L^2(\hh,\rho_\hh)$ is the regression function 
$f_\rho(x)=\int yd\rho(y|x)$ for $\rho_\hh$-almost every  $x\in\hh$. By considering Problem~\eqref{expmin} we are restricting the 
search for a solution to linear functions. Note that, since $\hh$ is in general  infinite dimensional, 
the minimum in~\eqref{expmin} might not be achieved.   Indeed, bounds on the above error measures depend on if, and how well, the regression function can be linearly 
approximated. The following assumption quantifies in a  precise way such a requirement. 
\begin{ass}
\label{ass:uno}
Consider the space
$
{\cal L}_\rho=\{f:\hh\to \R  ~|~ \exists w\in \hh~~ \text{with}~~ f=\scal{w}{\cdot}~\rho_\hh\text{- a. s.} \}, 
 $
and let $\overline {\cal L}_\rho$ be its closure in $L^2(\hh,\rho_\hh)$. Moreover, consider  the operator
\begin{equation}\label{eq:defL}
L:L^2(\hh,\rho_\hh)\to L^2(\hh,\rho_\hh),  \quad Lf(x)=\int \scal{x}{x'}f(x')d\rho(x'),
\quad \forall f \in L^2(\hh,\rho_\hh).
\end{equation}
Define
$ g_\rho=\argmin_{g\in \overline{{\cal L}_\rho}} \nor{f_\rho-g}_{\rho}$.
Ley $r \in \left[0,+\infty\right[$, and assume that 
\begin{equation}
\label{source}
(\exists g\in L^2(\hh,\rho_\hh))\quad \text{such that }\quad g_\rho=L^rg.
\end{equation}  
\end{ass}
The above assumption is standard \cite{SteiChri08}.  
Since  its statement is somewhat technical, and we provide a more general formulation in a  Hilbert space with respect to the usual RKHS setting, 
we further comment on its interpretation. We begin noting that   ${\cal L}_\rho$ is the space of  linear functions indexed by $\hh$ and
 is a proper subspace of $L^2(\hh,\rho_\hh)$, if Assumption~\ref{ass:zero} holds. Moreover,  under the same assumption,  it is easy to see that the operator $L$ is linear, self-adjoint, positive definite and trace class, hence compact, so that its fractional power in~\eqref{eq:defL} is well defined. It can be shown fairly easily that  the space ${\cal L}_\rho$ can be characterized in terms of the operator $L$, namely
\begin{equation}
\label{eq:mercer}
{\cal L}_\rho=L ^{1/2} \left(L^2(\hh,\rho_\hh)\right).
\end{equation}
This last observation allows to  provide an interpretation of Condition~\eqref{source}.
Indeed, given \eqref{eq:mercer}, for $r=1/2$, Condition 
\eqref{source} states that $g_\rho$ belongs to ${\cal L}_\rho$, rather than its closure.
In this case,  Problem~\ref{expmin} has at least one solution, and the  set $\cal O$ in~\eqref{eq:wdag} is not empty. 
Vice versa, if $\mathcal{O}\neq\varnothing$ then  $g_\rho\in \cal{L}_\rho$, and $w^\dag$ is well-defined.  
If $r>1/2$ the condition is stronger than for $r=1/2$, since the images of $L^2(\hh, \rho_\hh)$ with respect to $L^r$ are nested subspaces for increasing $r$\footnote{ 
If $r<1/2$  then the regression function does not have a best linear approximation since $g_\rho\notin {\cal L}_\rho$, and in particular, 
for $r=0$ we are not making any assumption. 
Intuitively, for  $0<r<1/2$, the condition quantifies {\em how far} $g_\rho$ is from ${\cal L}_\rho$, that is to be well approximated by a linear function.}.

\subsection{Iterative Incremental Regularized Learning}\label{sec:ite}
The learning algorithm we consider is defined by the following iteration.
\begin{mdframed}
Let $w_0\in\hh$ and $\gamma\in\mathbb{R}_{++}$. Consider the sequence $(\hat{w}_t)_{t\in\N}$ generated 
through the following procedure: given $t\in\N$ and $\hat{w}_t\in\hh$, define 
\vspace{-0.2cm}
\begin{align} \label{eq:nine1} 
\hat{w}_{t+1}&=\hat{u}^n_t,
\end{align}
where $\hat{u}^n_t$ is obtained at the end of one cycle, namely as the last step of the recursion 
\begin{equation}\label{eq:nine}
\hat{u}^{0}_t=\hat{w}_{t};\qquad \hat{u}^i_t=\hat u^{i-1}_{t}- \frac{\gamma}{n}(\scalh{\hat u^{i-1}_{t}}{x_i}-y_i)x_i, \quad i=1,\ldots,n.
\end{equation}
\vspace{-0.4cm}
\end{mdframed}
Each cycle, called an epoch, corresponds to one pass over data. 
The above iteration can be seen as the  incremental gradient method \cite{Ber97,nedic2001incremental} for  the minimization of the  empirical risk corresponding to $\tr$, that is the functional
\begin{equation}
\label{empmin}
\emp(w)=\frac 1 n \sum_{i=1}^n (\scalh{w}{x_i}-y_i)^2.
\end{equation}
(see also Section~\ref{sec:samp_appr}).Indeed, there is a vast literature on how the iterations ~\eqref{eq:nine1},~\eqref{eq:nine} can be used to minimize the empirical risk 
\cite{Ber97,nedic2001incremental}. 
Unlike these studies in this paper we are interested
in  how the iterations~\eqref{eq:nine1},~\eqref{eq:nine} can be used to approximately minimize the risk $\EE$.
The key idea is that while $\hat{w}_{t}$ is close to  a minimizer of the empirical risk when $t$ is sufficiently large,
a good  approximate solution of  Problem~\eqref{expmin} can be found by terminating the iterations earlier 
(early stopping). The analysis in the next few sections grounds theoretically this latter intuition. 

  
\section{Early stopping for incremental iterative regularization}
\label{sec:main}
In this section we present and discuss the main results of the paper, together with a sketch of the proof. 
The complete proofs can be found in Appendix~\ref{sec:proofs}.
We first present convergence results and then finite sample bounds 
for the norm and the excess risk. 
\begin{theorem} 
\label{thm:mainnorate} 
In the setting of Section~\ref{sec:setting}, let Assumption~\ref{ass:zero} hold. 
Let $\gamma\in\left]0,\kappa^{-1}\right]$. Then the following hold:
\begin{enumerate}
\item
\label{thm:mainnoratei}
If we choose a stopping rule $t^*\colon\NN^*\to\NN^*$ such that
\begin{equation}
\label{eq:stoprule1}
\lim_{n\to+\infty} t^*(n)=+\infty\quad\text{and}\quad \lim_{n\to+\infty} \frac{t^*(n)^3\log n}{n}=0
\end{equation}
then
\begin{equation}
\lim_{n\to+\infty} \EE(\hat{w}_{t^*(n)})-\inf_{w\in\hh} \EE(w) =0 \quad \P\text{-almost surely}.
\end{equation}
\item
\label{thm:mainnorateii}
Suppose additionally  that the set $\mathcal{O}$ of minimizers of \eqref{expmin} is nonempty and let 
$w^\dag$ be defined as in \eqref{eq:wdag}.
If we choose a stopping rule $t^*\colon\NN^*\to\NN^*$ satisfying the conditions in \eqref{eq:stoprule1} 
then
\begin{equation}
\label{eq:uconsh}
\norh{\hat{w}_{t^*(n)}-w^\dagger}\to 0 \quad\text{$\P$-almost surely}.
\end{equation}
\end{enumerate}
\end{theorem}

The above result shows that for an a priori fixed step-sized,
consistency is achieved computing  a suitable number $t^*(n)$ of iterations of algorithm
\eqref{eq:nine1}-\eqref{eq:nine} given $n$ points. 
The number of required iterations tends to infinity as the number of available training points increases. Condition \eqref{eq:stoprule1} 
 can be interpreted as an early stopping rule,  since it  requires  the  number of epochs not to grow too fast. 
 In particular, this excludes  
 the choice $t^*(n)=1$ for all $n\in\NN^*$, namely considering only one pass over the data. Given the currently known results, the failure of 
 a single pass does not appear to be surprising, considering the stepsize is fixed (not depending on $n$) and we are not averaging.
 In the following remark we make clear that, if we let the step size to depend on the length of one epoch, convergence is recovered  also for one pass. 
\begin{remark}[Recovering Stochastic Gradient descent]
Note that if in Theorem~\ref{thm:mainnorate} we let $\gamma$ to depend on $n$, than we can choose $t^*(n)=1$.
Indeed, choosing $\gamma(n)=\kappa^{-1}n^{\alpha}$, with $\alpha<1/4$, which coresponds to a stochastic
gradient method with step-size $\kappa^{-1}n^{\alpha-1}$, we can derive almost sure convergence 
of $\EE(\hat{w}_1)-\inf_\hh\EE$ as $n\to+\infty$ relying on the same proof.
\end{remark}

To derive finite sample bounds  we need to impose additional conditions. 
We will see that the behavior of the bias of the estimator depends on the smoothness assumption \eqref{source}.
We are in position to state our main result, giving a  finite sample bound.
\begin{theorem}[Finite sample bounds in $\hh$]
\label{thm:mainH}
In the setting of Section~\ref{sec:setting}, let $\gamma\in\left]0,\kappa^{-1}\right]$ for every $t\in\NN$.
Suppose that Assumption~\eqref{ass:uno} is satisfied for some $r\in\left]1/2,+\infty\right[$. 
Then the set $\mathcal{O}$ of minimizers of \eqref{expmin} is nonempty, and $w^\dagger$ in \eqref{eq:wdag} is
well defined. 
Moreover, the following hold:
\begin{enumerate} 
\item
\label{thm:mainHi}
 There exists $c\in\left]0,+\infty\right[$
such that,  for every $t\in\NN^*$, with probability greater than $1-\delta$,  
\begin{align}
\label{eq:boundprobh}
\norh{\hat{w}_t-w^\dagger}\leq 
\frac{32\log(16/\delta)}{\sqrt{n}}\!\left(\! M {\kappa^{-1/2}}+{2M^2}{\kappa^{-1}}\right.&\left.+3\|g\|_\rho\kappa^{r-3/2}\right)t\\
&+\left(\!\frac{r-1/2}{\gamma}\right)^{r-1/2}\!\!\!\|g\|_\rho t^{1/2-r}.
\end{align}
\item 
\label{thm:mainHii}
For the stopping rule $t^*:\NN^*\to\NN^*$ defined by
$ t^*(n)=
\big\lceil n^{\frac{1}{2r+1}}\big\rceil$,
with probability greater than $1-\delta$,  
\begin{align}
\label{eq:rcorh}
\nonumber\norh{\hat{w}_{t^*(n)}-w^\dagger}&\leq \left[{32\log\left(\frac{16}{\delta}\right)}\left( M {\kappa^{-1/2}}+{2M^2}{\kappa^{-1}}\right.+3\|g\|_\rho\kappa^{r-3/2}\right)\hspace{2cm} \\
&\hspace{3.5cm}+\left(\frac{r-1/2}{\gamma}\right)^{r-1/2}\|g\|_\rho\bigg] n^{-\frac{r-1/2}{2r+1}}.
\end{align}
\end{enumerate}
\end{theorem}
Note that the dependence on $\kappa$ suggests that a big $\kappa$, which corresponds to a small $\gamma$, helps in decreasing the sample error, 
but increases the approximation error. 
Next we present  the result for the excess risk. We consider only the attainable case, that is the case $r>1/2$ in Assumption~\ref{ass:uno}.
The case $r\leq 1/2$ is deferred to Appendix~\ref{app:nonatt}, since both the proof and the statement are conceptually similar to the attainable case.
 
\begin{theorem}[Finite sample bounds for the risk -- attainable case]
\label{thm:mainrhob}
In the setting of Section~\ref{sec:setting}, let Assumptions~\ref{ass:zero}  holds, and let $\gamma\in\left]0,\kappa^{-1}\right]$.
Let Assumption~\ref{ass:uno} be satisfied for some $r\in\left]1/2, +\infty\right]$.
Then the following hold:
\begin{enumerate}
\item
\label{thm:mainrhobi}
For every $t\in\NN^*$, with probability greater than $1-\delta$,  
\begin{equation}
\label{eq:boundprobr2}
\EE(\hat{w}_t)-\inf_\hh\EE\leq  \frac{2\big(32\log(16/\delta)\big)^2}{n}\left[M+2M^2\kappa^{-1/2}+3\kappa^{r}\|g\|_\rho\right]^2t^2+2\bigg(\dfrac{r}{\gamma t}\bigg)^{2r} \|g\|_\rho^2
\end{equation}
\item
\label{thm:mainrhobii}
For the stopping rule $t^*\colon\NN^*\to\NN^*\;$ $t^*(n)=\big\lceil n^{\frac{1}{2(1+r)}}\big\rceil$, 
with probability greater than $1-\delta$,  
\begin{align}
\label{eq:rcor2}
\nonumber \EE(\hat{w}_{t^*(n)})-\inf_{\hh}\EE \leq \left[8\left(\!32\log\frac{16}{\delta}\right)^{\!2}\right.&\left(M+2M^2\kappa^{-1/2}+3\kappa^{r}\|g\|_\rho\!\right)^2 \\
&+\left.2\bigg(\dfrac{r}{\gamma}\bigg)^{2r} \|g\|_\rho^2\right]n^{-r/(r+1)}
\end{align}
\end{enumerate}
\end{theorem}

Equations \eqref{eq:boundprobh} and \eqref{eq:boundprobr2} arise from a form of bias-variance 
(sample-approximation) decomposition of the error. 
Choosing the number of epochs that optimize the bounds in  \eqref{eq:boundprobh} and \eqref{eq:boundprobr2}
we derive a priori stopping rules and corresponding bounds~\eqref{eq:rcorh} and \eqref{eq:rcor2}. 
Again, these  results confirm that the number of epochs 
acts as a regularization parameter and the best choices following from equations~\eqref{eq:boundprobh} and
\eqref{eq:boundprobr2} suggest multiple passes over the data to be beneficial. 
In both cases, the  stopping rule depends on the smoothness parameter $r$ which is typically unknown, 
and hold-out cross validation is often  used  in  practice. Following \cite{CapYao06}, it is possible to show that this 
procedure allows to adaptively achieve the same convergence rate as in~\eqref{eq:rcor2}.

\subsection{Discussion}
We  discuss  a few comments and comparisons.
In Theorem~\ref{thm:mainH}, the obtained bound can be compared to known  
lower bounds, as well as to previous results for least squares 
algorithms obtained under Assumption~\ref{ass:uno}. 
Minimax lower bounds and individual lower bounds \cite{Caponnetto:2006,SteinwartHS09},
suggest that, for $r>1/2$, $O(n^{(r-1/2)/(2r+1))}$ is the optimal capacity-independent bound for
the $\hh$ norm\footnote{In a recent manuscript, it has been proved that these are indeed 
minimax lower bounds (G. Blanchard, personal communication)}. In this sense, Theorem~\ref{thm:mainH} provides  sharp bounds 
on the iterates. 
Bounds can  be improved only under stronger  assumptions, e.g.  on the covering numbers or 
on the eigenvalues of $\LK$ \cite{SteiChri08}. This question is left for future work.
The lower bounds for the excess risk \cite{Caponnetto:2006,SteinwartHS09} are of the form
$O(n^{-2r/(2r+1)})$ and in this case the results in Theorems~\ref{thm:mainrho} and \ref{thm:mainrhob} 
are not sharp, and in principle could be improved.
Our results can be contrasted with online learning algorithms mentioned in the introduction
that use step-size as regularization parameter.
Optimal capacity independent bounds are obtained in \cite{yiming}, see also \cite{TarYao14} and indeed such results
can be further improved considering capacity assumptions, see \cite{BacDie14} and references therein.

Interestingly, our results can also be contrasted with non incremental iterative regularization
approaches \cite{ZhangYu03, yao,bauer,BlaKra10,CapYao06,raskutti}. 
Our current results show that incremental iterative regularization, with distribution independent step-size, 
behaves as a full gradient descent, at least in terms of iterates convergence. We believe that proving additional advantages 
of incremental regularization over the batch one is an interesting
future research direction.
Finally, we note that optimal capacity independent and dependent bounds are known for several 
least squares algorithms including Tikhonov regularization see e.g. 
\cite{SteinwartHS09} and references therein,  
as well as for a larger class of so called  spectral filtering methods \cite{bauer,CapYao06}. 
These algorithms can be seen to be  essentially equivalent from a statistical perspective but different from
a computational point of view.

\subsection{Elements of  the proof}
The proofs of the main results are based on a decomposition of the error to
be estimated in two terms. The idea is  to build an auxiliary sequence and 
to majorize the error with the sum of two quantities that can be interpreted as 
a sample and an approximation error, respectively.  
Bounds on these two terms are then provided. The main technical contribution of the
paper is the sample error bound. The difficulty in proving this result is due
to the fact that multiple passes over the data induce complex statistical dependences 
in the iterates.

\paragraph{Error decomposition.}
We consider an auxiliary iteration $({w}_t)_{t\in\N}$ which is the expectation of the iterations 
 \eqref{eq:nine1} and \eqref{eq:nine}, starting from $w_0\in\hh$ with step-size $\gamma\in\R_{++}$. 
 More explicitly, given ${w}_t\in\hh$, the considered iteration generates  ${w}_{t+1}$ according to
\begin{align} \label{eq:nine1exp} 
{w}_{t+1}&={u}^n_t,
\end{align}
where ${u}^n_t$ is given by
\begin{align}\label{eq:nineexp}
{u}^{0}_t&={w}_{t};\qquad {u}^i_t= u^{i-1}_{t}- \frac{\gamma}{n}\int_{\hh\times\R}\left(\scalh{u^{i-1}_t}{x}-y\right) x\, d\rho(x,y)\,.
\end{align} 

If we let $S\colon\hh\to L^2(\hh,\rho_\hh)$ be the linear  map $w\mapsto \scalh{w}{\cdot}$,
which is bounded by $\sqrt{\kappa}$ under Assumption~\eqref{ass:zero} holds, then 
it is well-known that
\begin{align}\label{eq:dec}
\nonumber (\forall t\in\N)\quad \EE(\hat w_t)-\inf_\hh\EE&=\nor{S\hat{w}_t-g_\rho}_\rho^2\leq 2 \nor{S\hat{w}_t-Sw_t}_\rho^2 +2 \nor{Sw_t-g_\rho}_\rho^2\\
&\leq 2\kappa \norh{\hat{w}_t-w_t}^2+2(\EE(w_t)-\inf_\hh\EE).
\end{align}
In this paper, we refer to the gap between the empirical  and the expected iterates $\norh{\hat{w}_t-w_t}$ as the {\em sample error}, and to 
$\mathcal{A}(t,\gamma,n)=\EE(w_t)-\inf_\hh\EE$ as the {\em approximation error}.
Similarly, if $w^\dagger$ (as defined in \eqref{eq:wdag}) exists, using the triangle inequality, we obtain
\begin{align}\label{eq:decw}
 \norh{\hat{w}_t-w^\dagger}&\leq \norh{\hat w_t-w_t}+\norh{w_t-w^\dagger}.
\end{align}
\paragraph{Proof main  steps.}
In the setting of Section~\ref{sec:setting}, we summarize the key steps to derive a general bound for the sample error.
Indeed, this is the main technical contribution of the paper. The proof of the behavior of the approximation error is more standard. The bound on the sample error is derived through many technical lemmas  and uses concentration
inequalities applied to martingales. Its complete derivation is reported in Section~\ref{sec:samp_appr}. 
We underline that the crucial point is the probabilistic inequality in {\bf STEP 5} below. 
We need to introduce the additional linear operators:
$T\colon \hh\to \hh\colon\; T=S^*S,$
$(\forall x\in \XX)\; S_x\colon\hh\to \mathbb{R}\colon S_xw=\langle w,x\rangle$, and 
$T_{x}\colon \hh\to \hh\colon \; T_x=S_xS^*_x. $
Moreover, set $\hat{T}=\sum_{i=1}^n T_{x_i}/n$. We are now ready to state the main steps of the proof of the main results.\\
{\bf Sample error bound (STEP 1 to 5)}
\ \\
{\bf STEP 1 (see Proposition~\ref{prop:miter}):} Find equivalent formulations for the sequences $\hat{w}_t$  and $w_t$:
\begin{align*}
&\hat{w}_{t+1}=(I-\gamma\hat{T})\hat{w}_t + {\gamma}\bigg(\frac{1}{n} \sum_{j=1}^n S^*_{x_j}y_j\bigg)+{\gamma^2}\left( \hat{A}\hat{w}_t- \hat{b}\right)\\
&{w}_{t+1}=\left(I-{\gamma}T\right){w}_t + {\gamma} S^*g_\rho+{\gamma^2} (Aw_t-b),
\end{align*}
with
\begin{align*}
&\hat{A}=\frac{1}{n^2}\sum_{k=2}^n \left[\prod_{i=k+1}^n \left(I-\frac{\gamma}{n} T_{x_i} \right)\right]  T_{x_k}\sum_{j=1}^{k-1} T_{x_j},
&& \hat{b}=\frac{1}{n^2}\sum_{k=2}^n \left[\prod_{i=k+1}^n \left(I-\frac{\gamma}{n} T_{x_i} \right)\right]  T_{x_k}\sum_{j=1}^{k-1} S^*_{x_j}y_j.\\
&A=\frac{1}{n^2}\sum_{k=2}^n \left[\prod_{i=k+1}^n \left(I-\frac{\gamma}{n} T\right)\right]T\sum_{j=1}^{k-1} T, 
&&b=\frac{1}{n^2}\sum_{k=2}^n \left[\prod_{i=k+1}^n \left(I-\frac{\gamma}{n} T\right)\right]T\sum_{j=1}^{k-1}S^*g_\rho.
\end{align*}
{\bf STEP 2 (see Lemma~\ref{prop:diff}):} Use the formulation obtained in {\bf STEP 1} to derive the following recursive inequality
\begin{equation}
\label{eq:diffsk}
\hat{w}_{t}-w_{t}=\Big(I-\gamma \hat{T}+\gamma^2 \hat{A}\Big)^t (\hat{w}_0-w_0)+\gamma\sum_{k=0}^{t-1}\left(I-\gamma\hat{T}+\gamma\hat{A}\right)^{t-k+1} {\zeta}_{k} 
\end{equation}
with 
\begin{equation}
\label{eq:zetaksk}
{\zeta}_{k}=(T-\hat{T}){w}_k +\gamma(\hat{A}-A){w}_k+\bigg( \frac{1}{n}\sum_{i=1}^n\hat{S}^*_{x_i}y_i-S^*g_\rho\bigg)+\gamma(b- \hat{b}).
\end{equation}
{\bf STEP 3 (see Lemmas~\ref{ref:norma} and \ref{prop:diffb}):}  Initialize $\hat{w}_0=w_0=0$, prove that $\|I-\gamma\hat{T}+\gamma\hat{A}\|\leq 1$ and derive from {\bf STEP 2} that
\[
\norh{\hat{w}_{t}-w_{t}}\leq  \gamma\big( \|T-\hat{T}\|+ \gamma\|\hat{A}-A\|\big)\sum_{k=0}^{t-1} \|w_k\|_\hh+\gamma t\Big(\big\|\frac{1}{n}\sum_{i=1}^n\hat{S}^*_{x_i}y_i-S^*g_\rho\big\| + \gamma \|b- \hat{b}\| \Big).\]
{\bf STEP 4 (see Lemma~\ref{lem:ftbound}):} Let Assumption~\ref{ass:uno} hold for some $r\in\R_+$ and $g\in L^2(\hh,\rho_\hh)$. Prove that
\[
(\forall t\in\hh)\quad\|w_t\|_\hh\leq 
\begin{cases}
\max\{\kappa^{r-1/2}, (\gamma t)^{1/2-r}\}\|g\|_\rho &\text{if  $r\in\left[0,1/2\right[$}\\
\kappa^{r-1/2}\|g\|_\rho &\text{if  $r\in\left[1/2,+\infty\right[$}\\
\end{cases}
\] 
{\bf STEP 5 (see Lemma~\ref{prop:pin} and Proposition~\ref{lem:ak-ak}:} Prove that with probability greater than $1-\delta$ the following inequalities hold:
\begin{align*}
&\|\hat{A}-A\|_{HS} \leq \frac{32\kappa^2}{3\sqrt{n}} \log\frac4 \delta, && \norh{\hat{b}-b}\leq \frac{32\kappa M^2}{3\sqrt{n}} \log\frac4 \delta\\
&\Big\|\frac{1}{n}\sum_{i=1}^n T_{x_i}-T\Big\|_{HS}\leq \frac{16\kappa}{3\sqrt{n}}\log\frac 2 \delta,
&&\Big\|\frac{1}{n}\sum_{i=1}^n S^*_{x_i}y_i-S^*g_{\rho}\Big\|_{\hh}\leq \frac{16\sqrt{\kappa} M}{3\sqrt{n}}\log\frac 2 \delta
\end{align*}

{\bf STEP 6 (approximation error bound, see Theorem~\ref{thm:apprerr}):} Prove that, if Assumption~\ref{ass:uno} holds for some $r\in\left]0,+\infty\right[$,
then 
\[
\EE(w_t)-\inf_\hh \EE\leq  \bigg(\dfrac{r}{\gamma t}\bigg)^{2r} \|g\|_\rho^2\,.
\]
Moreover, if Assumption~\ref{ass:uno} holds with $r=1/2$, then $\|w_t-w^\dag\|_\hh\to 0$, and 
if Assumption~\ref{ass:uno} holds for some $r\in\left]1/2,+\infty\right[$, then
\[
\|w_t-w^\dag\|_\hh \leq \bigg(\dfrac{r-1/2}{\gamma t}\bigg)^{r-1/2} \|g\|_\rho\,.
\]
 {\bf STEP 7:} Plug the sample and approximation error bounds obtained in {\bf STEP 1-5} and {\bf STEP 6} in \eqref{eq:dec} and \eqref{eq:decw}, respectively.

\section{Experiments}
\label{sec:exp}
\paragraph{Synthetic data.}
We consider a linear regression problem with random design in $\R^d\times \R$. The input points $(x_i)_{1\leq i\leq n}$ are uniformly distributed 
in $[0,1]$ and the output points are obtained as $y_i=\langle w^*, \Phi(x_i)\rangle+N_i$, where $N_i$ is a gaussian noise with zero mean and standard deviation 1 and $\Phi=(\varphi_k)_{1\leq k\leq d}$ 
is a dictionary of trigonometric functions whose $k$-th element is $\varphi_k(x) = \cos((k-1)x)+\sin((k-1)x)$. 
In Figure \ref{fig:synthdataexp}, we plot the test error for $d=5$ (with $n=80$ in (a) and $800$ in (b)).
The plots show that the number of the epochs acts as a regularization parameter, and that early stopping is beneficial to achieve a better test error. Moreover, according to our theoretical findings, the experimental results suggest that the number of performed epochs should increase if the number of available training points increases.
\vspace{-0.2cm}
\begin{figure}[h!]
        \centering
        \subfigure[]{
                \includegraphics[trim=0cm 0cm 2cm 0cm, clip=true, width=0.45\textwidth]{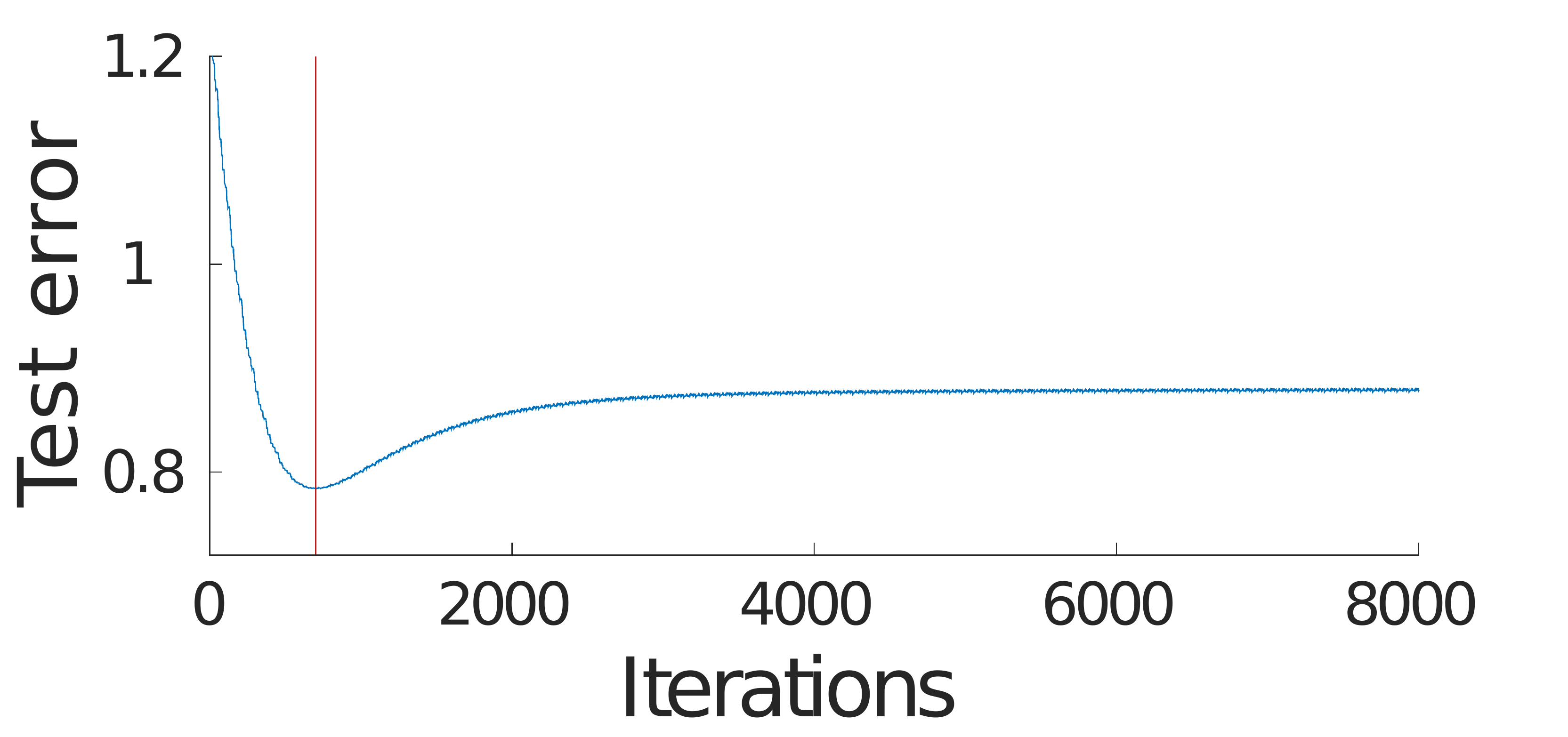}
        }
        \subfigure[]{
                \includegraphics[trim=0cm 0cm 2cm 0cm, clip=true, width=0.45\textwidth]{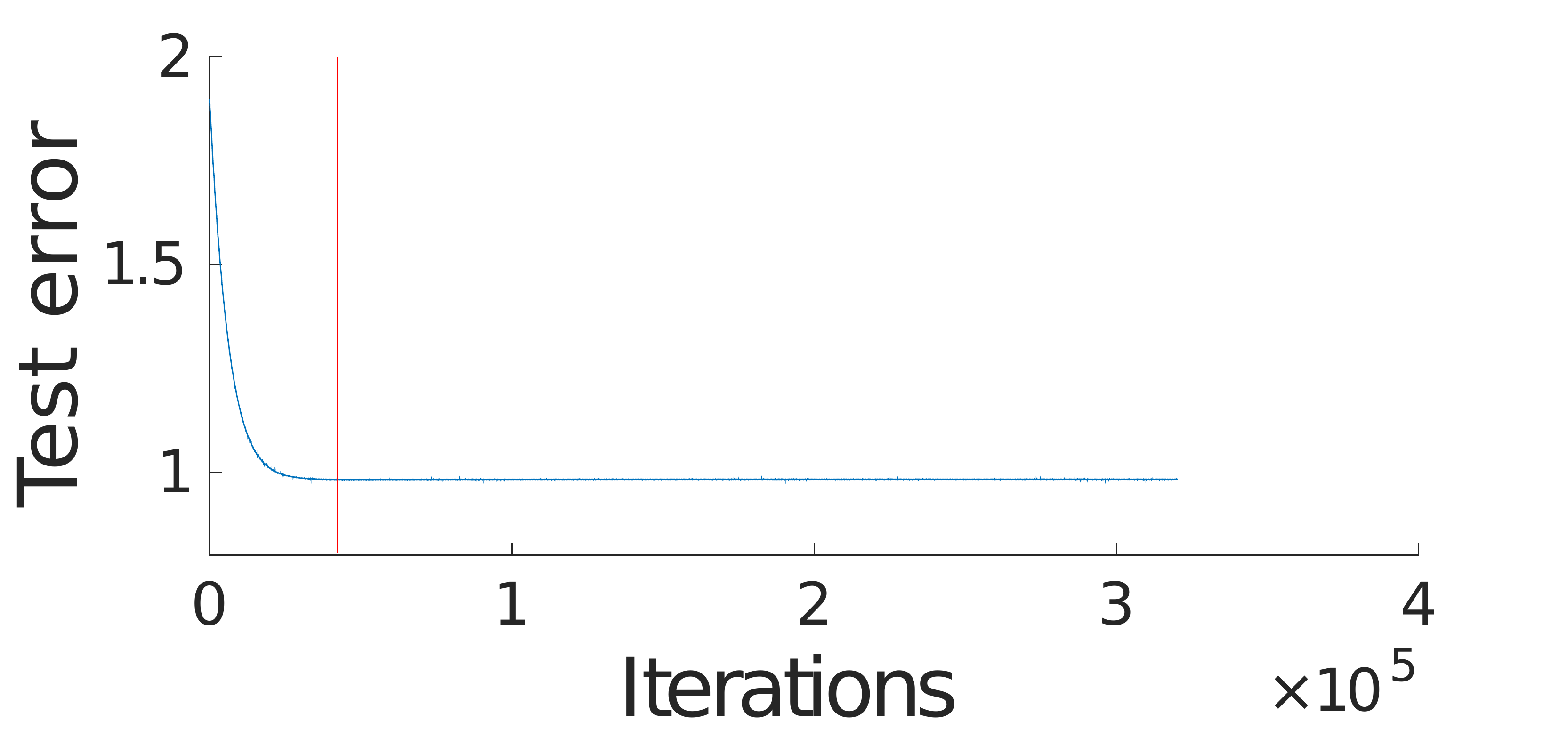}
        }
		\caption{Evolution of the test error with the number of iterations. In (a), $n=80$, and total number of iterations of IIR is 8000, corresponding to 100 epochs. 
		In (b), $n=800$ and the total number of epochs is 400. The best test error is obtained for 9 epochs in (a) and for 31 epochs in (b). } 
        \label{fig:synthdataexp}
\end{figure}

\paragraph{Real data.} We tested the kernelized version of our algorithm (see Appendix~\ref{sec:somspe}) on the \texttt{cpuSmall}\footnote{Available at \url{http://www.cs.toronto.edu/~delve/data/comp-activ/desc.html}}, \texttt{Adult} and \texttt{Breast Cancer Wisconsin (Diagnostic)}\footnote{\texttt{Adult} and \texttt{Breast Cancer Wisconsin (Diagnostic)},  {UCI}  repository, 2013.} real-world datasets. We considered a subset of \texttt{Adult}, with $n=1600$. The results are shown in Figure \ref{fig:realdataexp}. A comparison of the test errors obtained with the kernelized version of the method proposed in this paper (Kernel Incremental Iterative Regularization (KIIR)), Kernel Iterative Regularization (KIR), that is the kernelized version of gradient descent, and Kernel Ridge Regression (KRR) is reported in Table \ref{tab:realdataexp}. The results show that the proposed method achieve a  test error comparable  to that of KIR and KRR. 

\begin{figure}[h!]
        \centering
                \includegraphics[width=0.30\textwidth]{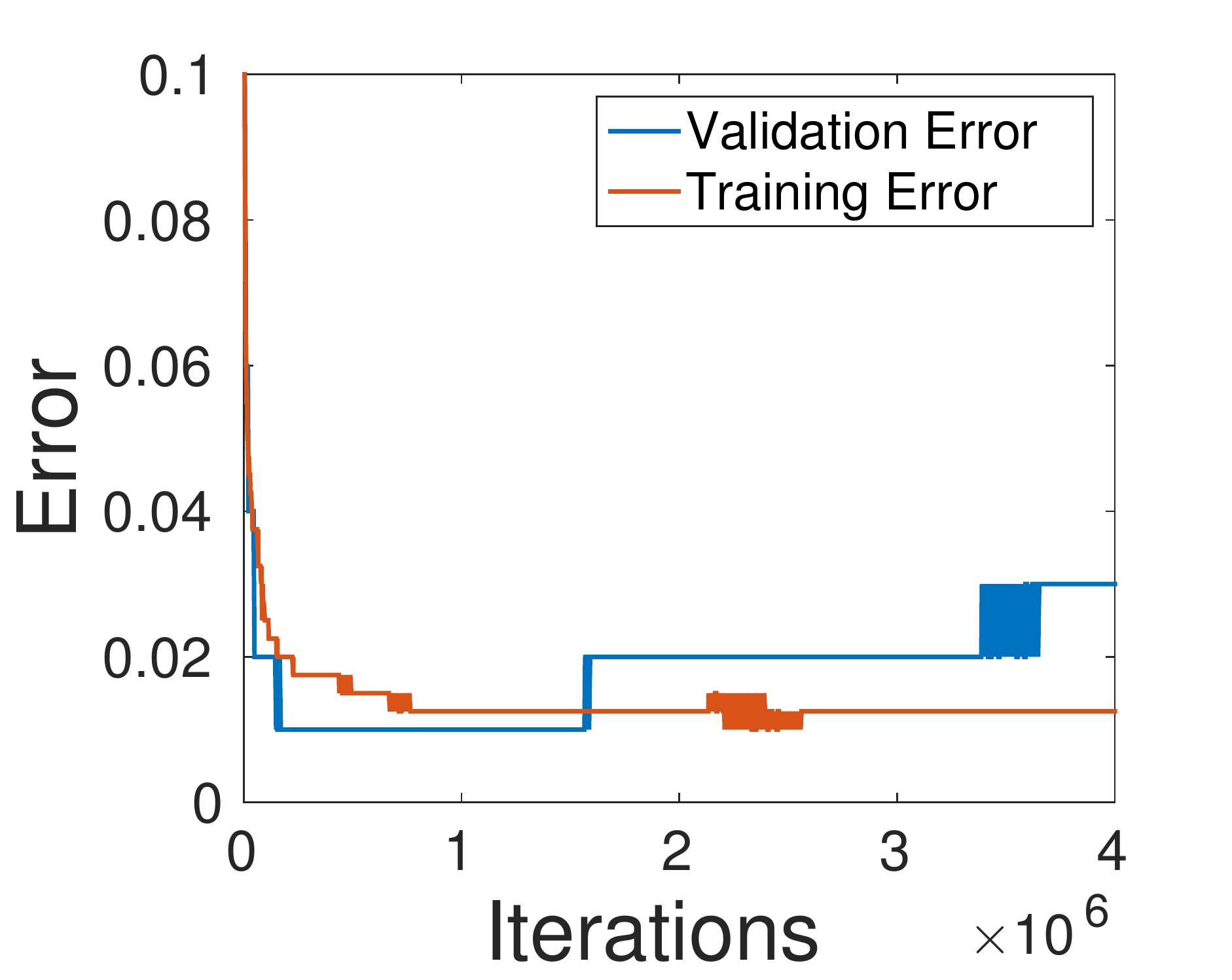}
        \caption{Training (orange) and validation (blue) classification errors obtained by KIIR on the \texttt{Breast Cancer} dataset as a function of the number of iterations over the training points. As can be seen by visual inspection, the test error increases after a certain number of iterations, while the training error is approximately decreasing with the number of iterations. }
        \label{fig:realdataexp}
\end{figure}
\vspace{-0.5cm}
\begin{table*}[h!]
\caption{Test error comparison on real datasets. Median values over 5 trials.}
\begin{center}
\scalebox{0.8}{
\begin{tabular}{|cccc||c||c|c|}
\hline 
\textbf{Dataset} & $\mathbf{n_{tr}}$ & $\mathbf{d}$ & \textbf{Error Measure} &  \textbf{KIIR} &\textbf{KRR}  & \textbf{KIR}\\ \hline 
\texttt{cpuSmall} & 5243 & 12 & RMSE & 5.9125 & 3.6841 & 5.4665 \\ 
\hline 
\texttt{Adult} & 1600 & 123 & Class. Err. & 0.167 & 0.164 &  0.154\\ 
\hline 
\texttt{Breast Cancer} & 400 & 30 & Class. Err. & 0.0118 & 0.0118 & 0.0237\\ 
\hline
\end{tabular}
}
\end{center}
\label{tab:realdataexp}
\end{table*}


\newpage
\appendix
\section{Some Special Cases of Interest}\label{sec:somspe}

Finally,  before proving our main results, we illustrate and discuss a few  special instances of the considered setting and related quantities.
\paragraph{Linear and Functional Regression}
In classical linear regression,  data are described by the following model
$$
y_i=w^T_*x_i+\delta_i, \quad i=1, \dots,n
$$
where $\delta_i$, $i=1, \dots,n$, are  i.i.d.  sample from a normal distribution and $w_*, x_1, \dots, x_n\in \R^d$, $d\in \N^*$.  In fixed design regression, the inputs $x_1, \dots, x_n$ are assumed to be fixed, while  in random design regression they are random sample according to some fixed unknown distribution \cite{SteiChri08}. It is easy to see that this latter setting is a special case of the framework in the paper (indeed the analysis in the paper can be also adapted with minor modifications  to the fixed design setting). The regression model can be further complicated assuming  the function of interest to be  non linear (while we might still restrict the search of a solution to linear estimators). This can be dealt with for example considering kernel methods as we discuss below. Another special case of the setting  in the paper is that of functional regression, where the input points are assumed to be infinite dimensional objects, for example curves, and they are formally described as functions in a Hilbert space.  Clearly also this  example is subsumed as a special case of our setting.
\paragraph{Learning with Kernels}
The setting in the paper reduces to nonparametric learning in RKHS as a special case. 
Let ${ \Xi}\times \R$ be a probability space with distribution $\mu$, that be can seen as the input/output space.  The goal is then to minimize the risk, that, 
considering the square loss function, is  given by 
\begin{equation}\label{eq:expriskf}
\EE(f)=\int (y-f(\xi))^2d\mu(\xi, y)
\end{equation}
and is  well defined for all measurable functions.   A common way to build an estimator is to consider  a  symmetric  kernel $K:{\Xi}\times {\Xi}\to \R$ which is positive definite, that is for which the matrix with entries $K(\xi_i,\xi_j)$, $i,j=1\dots n$, is positive semidefinite for all 
in $\xi_1, \dots, \xi_n\in \Xi$, $n\in \N^*$. Such a kernel defines a unique  Hilbert space of function $\hh_K$ with inner product $\scal{\cdot}{\cdot}_{K}$ and such that for all $\xi \in \Xi$, $K_\xi(\cdot)=K(\xi, \cdot)\in \hh_K$ and the following reproducing property holds for all $f\in \hh_K$, $f(\xi)=\scal{f}{K_\xi}_{K}$.
To see how this setting is subsumed by the one in the paper, it is useful to introduce the (feature) map $\Phi:\Xi\to \hh_K$, where $\Phi(\xi)=K_\xi$, for $\xi\in \Xi$ and further consider 
$\overline \Phi:\Xi\times \R \to \hh_K\times\R$, where $\overline \Phi(\xi,y)=(K_\xi,y)$, for $\xi\in \Xi$ and $y\in\mathbb{R}$.
Assuming the kernel to be measurable, we can  view $\overline \Phi$ as a random variable. If we denote its distribution on $\hh_K\times \R$ by $\mu_{\overline\Phi}$, then we can then let $\hh=\hh_K$ and $\rho=\mu_{\overline \Phi}$. It is known that the functions in a RKHS a measurable provided that the kernel is measurable \cite{SteiChri08}, hence if  we consider the risk of a function $f\in \hh_K$ we have
$$
\int_{\Xi\times\R} (y-f(\xi))^2d\mu(\xi, y)=
\int_{\Xi\times\R} (y-\scal{f}{K_\xi}_{K})^2d\mu(\xi, y)=
\int_{\hh\times\R} (y-\scal{f}{x})^2d\rho(x, y),
$$
where we made the change of variables $(K_\xi,y)=(x,y)$.
As is well known in machine learning, we can view learning a function using a kernel as learning a 
linear function in suitable Hilbert space.
\paragraph{Integral and Covariance Operators.}
The operator $L$ can be seen as an integral operator associated to a linear kernel and is closely related to the covariance operator, or rather the second moment  operator defined by $\rho$.
This connections allows to interpret Assumption~\ref{ass:uno} in terms of the principal components.

To see this note that, under Assumption~\ref{ass:zero} it is easy to see that  $S$ is bounded, and  its adjoint is given by 
$$
S^*:L^2(\hh,\rho_\hh)\to\hh, \quad S^* w = \int xf(x)d\rho_\hh(x), \quad \forall f\in  L^2(\hh,\rho_\hh).
$$
Then a straightforward calculation shows that $L=SS^*$. 
Moreover we can define $T:L^2(\hh,\rho_\hh)\to\hh$ as $T=S^*S$ and check that 
$$
T w = \int x\otimes x d\rho_\hh(x), \quad \forall w\in  \hh,
$$
where $x\otimes x=\scal{x}{\cdot}x$, for all $x\in \hh$.
The operator $T$ is the second moment operator associated to $\rho$ and its eigenfunctions are the principal components.  Under Assumption~\ref{ass:zero}, the operators $T,L$ are linear, positive, sef-adjoint and trace class, $S,S^*$ are bounded and Hilbert Schmidt, hence compact. 
The operators $T,L$ have the same non zero eigenvalues $(\sigma_j)_j$ which are the square of the singular values of $S$.  If we denote by $(v_j)_j$ the eigenfunctions of $T$, the eigenfunctions of $L$ can be chosen to be $(u_j)_j$ with 
$u_j(x)=\sigma_j^{-1}\scal{v_j}{x}_\hh$, $\rho_\hh$-almost surely. This latter observation allows an interpretation of Condition~\eqref{source}. By considering  higher fractional power we are  essentially assuming that the regression function can be linearly approximated and its approximation can be  effectively  represented considering the principal components associated to large eigenvalues.



\paragraph{Binary Classification}
The results in the paper can be directly applied to  binary classification. Indeed, in this  setting the outputs are binary valued i.e. $\{-1,1\}$ and the goal is to learn a classifier  $c:\hh \to \{-1,1\}$ with small misclassification risk 
\begin{equation}\label{eq:miscrisk}
R(c)=\P\left ( c(X)\neq Y\right).
\end{equation}
The above risk is minimized by the so called Bayes decision rule defined by  $b_\rho(x)=\text{sign}(2\rho(1|x)-1)$, $~\rho_H$- almost surely, and where for $a\in \R$, $\text{sign}(a)=1$, if $a\ge 1$ and $\text{sign}(a)=-1$ otherwise.
A relaxation approach is usually considered to  learn a classification rule, which is based on replacing the risk $R$ with convex error functional defined over real valued functions, e.g. 
considering~\eqref{eq:expriskf}. A classification rule is then obtained by taking the sign.

So called comparison results quantify the cost of the relaxation. In particular, it is known that the following inequality related $R$ and $\EE$ defined in~\eqref{eq:miscrisk}, ~\eqref{eq:expriskf} respectively,
$$
R(\text{sign}(f))-R(b_\rho) \le \sqrt{\EE(f)-\EE(f_\rho)} 
$$
for all measurable functions $f$. The latter inequality allows to derive  excess misclassification risk and can be improved under additional assumption. We refer to \cite{yao}, for further details in this direction.
%
%
%


\section{Proofs}
\label{sec:proofs}
In this appendix we prove the main results. The proof is quite long, and
will be given relying on a  series of lemmas. 
\subsection{Preliminary Results}
\label{sec:prelim}


We collect very general results that will be applied to our setting.

Let $\mathcal{B}$ be a normed space. For every $r\in\N$, let $A_r\colon\mathcal{B}\to \mathcal{B}$ be a linear operator, 
let $(B_r)_{r\in\N}$ be a sequence in $\mathcal{B}$, and define the sequence $(X_r)_{r\in\mathbb{N}}$ in $\mathcal{B}$
recursively as
\begin{equation}\label{eq:dftv}
X_{r+1}=A_rX_r+B_r.
\end{equation}
We repeatedly use the following well-known equality, which is valid for every $r\in\N^*$
and for every integer $s\leq r$,
\begin{equation}\label{eq:dftvf}
X_{r}=\left(\prod_{i=s}^{r-1}A_i \right)X_s+\sum_{k=s}^{r-1}\left(\prod_{i=k+1}^{r-1}A_i\right) B_k\,.
\end{equation}

We next state an auxiliary lemma, establishing the minimizing property of the gradient descent iteration also when the infimum
is not attained. Despite the result is a basic property of a  very classical algorithm, we were not able to find the proof of this fact.
Convergence properties of gradient descent are usually studied in two settings: for differentiable functions (not necessarily convex)
and for convex functions. In the first case, the typical results do not assume existence of a minimizer and establish convergence
to zero of the gradient of the function \cite{Pol87}. In the convex setting, the minimizing property is established assuming the existence of a 
minimizer \cite{Pol87}. 

 \begin{lemma}
 \label{lem:grad}
 Let $\hh$ be a Hilbert space, and  $F\colon\hh\to\mathbb{R}$ be a convex and differentiable function 
 with $\beta$-Lipschitz continuous gradient. Let $v_0\in\hh$, let $(\eta)_{k\in\NN}$ be such
 that, for every $k\in\NN$, $\eta_k\in\left]0,2/\beta\right[$ and define, for every $k\in\mathbb{N}$, $v_{k+1}=v_k-\eta_k \nabla F(v_k)$. 
 Then
 \begin{equation}
 (\forall u\in\mathcal{H}) \quad F(v_k)-F(u)\leq \frac{\|u-v_0\|_\hh}{2\sum_{j=0}^k\eta_j}
 \end{equation}
 In particular,  if $\sum_{k\in\NN}\eta_k=+\infty$, $F(v_k)\to \inf F$.
 \end{lemma}
 \begin{proof} Since $F$ is convex and differentiable, 
 \begin{align}
 \nonumber (\forall k\in\NN)(\forall u\in\hh) \quad F(u)-F(v_k) &\geq \langle \nabla F(v_k), u-v_k \rangle_\hh \\
\label{eq:subc} &=\eta_{k}^{-1}\langle v_k-v_{k+1},u-v_k\rangle_\hh.
 \end{align}
  Therefore, 
 \begin{align}
 \nonumber (\forall k\in\NN) \quad 2\eta_{k}(F(u)-F(v_{k})&\geq -2\langle v_{k+1} -v_k,u-v_{k}\rangle_\hh\\
 \nonumber&=\|u-v_{k+1}\|_\hh^2-\|v_{k+1}-v_k\|_\hh^2-\|u-v_k\|_\hh^2\\
\label{eq:tel} &=\eta_k^2\|\nabla F(v_k)\|_\hh^2+\|u-v_{k+1}\|_\hh^2-\|u-v_k\|_\hh^2
 \end{align}
Let $t\in\NN$ and define $\sigma_t=\sum_{k=0}^t\eta_k$. Summing \eqref{eq:tel} for $k=0,\ldots,t$ we obtain
\begin{equation}
\label{eq:one}
2\sigma_t F(u) -2\sum_{k=0}^t \eta_k F(v_k) \geq \sum_{k=0}^t \eta_k^2\|\nabla F(v_k)\|_\hh+\|u-v_{t+1}\|_\hh^2-\|u-v_0\|_\hh^2.
\end{equation}
Using the Lipschitz continuity of the gradient of $F$  (see \cite[Equation (15) p.6]{Pol87}),
\[
(\forall k\in\NN)\quad F(v_k)-F(v_{k+1})\geq \eta_k\left(1-\frac{\eta_k\beta}{2}\right) \|\nabla F(v_k)\|_\hh^2
\]
Therefore,
\begin{equation}
\label{eq:su}(\forall k\in\NN)\quad \sigma_k F(v_k)-\sigma_{k+1}F(v_{k+1})+\eta_{k+1}F(v_{k+1})\geq \sigma_k \eta_k\left(1-\frac{\eta_k\beta}{2}\right) \|\nabla F(v_k)\|_\hh^2
\end{equation}
Summing \eqref{eq:su} for $k=0,\ldots,t-1$ we get, for every $t\in\NN$
\begin{equation}
\label{eq:summ}
-\sigma_{t}F(v_t)+\sum_{k=0}^t\eta_kF(v_k) \geq \sum_{k=0}^{t-1} \sigma_k \eta_k\left(1-\frac{\eta_k\beta}{2}\right) \|\nabla F(v_k)\|_\hh^2
\end{equation}
Adding \eqref{eq:summ} to \eqref{eq:one} we get, for every $u\in\hh$
\[
2\sigma_t (F(u)-F(v_t)) \geq \sum_{k=0}^t \eta_k^2\|\nabla F(v_k)\|_\hh+\|u-v_{t+1}\|_\hh^2-\|u-v_0\|_\hh^2+2 \sum_{k=0}^{t-1} \sigma_k \eta_k\left(1-\frac{\eta_k\beta}{2}\right) \|\nabla F(v_k)\|_\hh^2
\]
and hence,
\[
(\forall u\in\hh)\qquad F(v_t)-F(u) \leq \frac{\|u-v_0\|_\hh^2}{2\sigma_t}. 
\]
\end{proof}

We next recall a probabilistic inequality for martingales \cite[Theorem 3.4]{Pin94} 
(see also \cite[Lemma A.1 and Corollaries A.2 and A.3]{TarYao14}).
\begin{theorem}
\label{thm:pinelis} 
Let $(\xi_i,\mathcal{F}_i)_{1\leq i\leq n}$ be an adapted family of random vectors taking values in a Hilbert space with
norm $\|\cdot\|$, such that $\E[\xi_i|\mathcal{F}_{i-1}]=0$ a.s. Assume that there exist $M\in\R_{++}$ such that $\|\xi_i\|\leq M$. 
Then, for every $\delta\in \,]0,1[\,$ the following holds 
\[
\mathbb{P}\left(\left\{\sup_{1\leq j\leq n}\Big\|\frac{1}{n}\sum_{i=1}^j \xi_i\Big\|\leq \frac{8M}{3\sqrt{n}}\log\frac2\delta \right\}\right)\geq 1-\delta\,.
 \]
\end{theorem}
\subsection{Proof of STEP 1}
\label{sec:samp_appr}
Here we first introduce a recursive expression which is satisfied by the sequence
 $(\hat{w}_t)_{t\in\mathbb{N}}$, that allows to interpret the incremental gradient iteration 
 as a gradient descent iteration with errors.
To do so, we start by introducing some further notation and then show that the iteration presented in \eqref{eq:nine1}-\eqref{eq:nine}
results from the application of the incremental gradient method to the empirical risk. 

Let $x\in\hh$ and define $S_x\colon\hh\to \mathbb{R}$, by setting $S_x w=\scalh{w}{x}$. 
Then $S_x$ is a bounded linear operator and $\nor{S_x}\leq \norh{x}$. 
Using this linear operator, and the operator $S$ introduced in Assumption~\ref{ass:uno}, 
Problems \eqref{expmin} and \eqref{empmin} can be expressed as convex quadratic minimization problems. 
The empirical risk can be written as 
\begin{equation}
\label{eq:emp}
\emp(w)=\frac 1 n \sum_{i=1}^n (S_{x_i}w-y_i)^2,
\end{equation}
and recalling Assumption~\ref{ass:uno} and \eqref{eq:dec}, we have 
\begin{equation}
\label{eq:exp}
\EE(w)=\int_{\hh\times \R} (S w(x)-y)^2 d\rho(x,y)=\nor{Sw-g_\rho}_\rho^2+\inf_{\hh}\EE.
\end{equation}
Define the operators 
\begin{equation}
\label{eq:T}
T\colon \hh\to \hh\colon\; T=S^*S,\quad\text{and } \quad(\forall x\in \XX) \quad T_{x}\colon \hh\to \hh\colon \; T_x=S_xS^*_x. 
\end{equation}
Then, computing the gradients of $\EE$ and $\emp$ respectively, \eqref{eq:nine1}-\eqref{eq:nine} can be rewritten as
\begin{align}\label{eq:nineop}
\hat{u}^{0}_t&=\hat{w}_{t};\qquad \hat{u}^i_t=\hat u^{i-1}_{t}- \frac{\gamma}{n}(T_{x_i}\hat u^{i-1}_{t}-S_{x_i}^*y_i), \quad i=1,\ldots,n
\end{align}
and \eqref{eq:nine1exp}-\eqref{eq:nineexp} can be expressed as
\begin{align}\label{eq:nineexpop}
{u}^{0}_t&={w}_{t};\qquad {u}^i_t= u^{i-1}_{t}- \frac{\gamma}{n}(T u^{i-1}_{t}-S^*g_\rho)\,.
\end{align}  
It is apparent from \eqref{eq:nineop} that the considered iteration is derived from the application
of the incremental gradient algorithm to the empirical error (see \cite{Ber97,nedic2001incremental}).
At the same time \eqref{eq:nineexpop} shows that iteration \eqref{eq:nine1exp}-\eqref{eq:nineexp} can be seen as the 
result of applying the incremental gradient descent algorithm to the expected loss, 
which clearly, for fixed $n$, can be written as
$$w\mapsto\frac{1}{n}\sum_{i=1}^n {\EE(w)}\,.$$ 

\begin{lemma}
\label{lem:incrvsgrad}
Let $t\in\NN$, and let $w_t$ be defined as in \eqref{eq:nine1exp}-\eqref{eq:nineexp}, with
$w_0=0$. Let $\eta=\gamma/n$.   Then
\begin{equation}\label{eq:graddesc2}
w_{t}=\eta\sum_{j=0}^{nt-1} \big(I-\eta T\big)^{nt-j-1}S^*g_\rho .
\end{equation}
\end{lemma}
\begin{proof}  
Let, for every $k\in\NN$,
\begin{equation*}
 v_{k+1}=(I-\eta T)v_k+\eta S^*g_\rho. 
\end{equation*}
Then, by \eqref{eq:nineexpop}, $w_t=v_{nt}$, and \eqref{eq:graddesc2} follows. 
\end{proof}

In other words, the statement of Lemma~\ref{lem:incrvsgrad}  states that the $t$-th epoch 
of the incremental gradient descent iteration in \eqref{eq:nine1exp}-\eqref{eq:nineexp}. 
coincides with $n$ steps of gradient descent with stepsize $\gamma/n$.

Next, we relate the iteration \eqref{eq:nine1}-\eqref{eq:nine} to the
gradient descent iteration on the empirical error.  These will be used in the error analysis  
and provide some useful comparison between these two methods. 
Hereafter, $\hat{T}$ is the operator $\big(\sum_{i=1}^n T_{x_i}\big)/n$.

The following lemma provides an alternative expression for the composition of linear operators. 
\begin{lemma}
\label{lem:decprod} Let $n\in\N^*$, let  $(T_i)_{1\leq i\leq n}$ be a family of linear operators from $\hh$ to 
$\hh$, and let $(w_i)_{1\leq i\leq n}\subset \hh^n$. Then
 \begin{equation}
 \label{eq:decprod} 
 \prod_{i=1}^n (I-T_i)=I-\sum_{j=1}^nT_{j}+ \sum_{k=2}^n \bigg(\prod_{i=k+1}^n \left(I- T_i\right)\bigg)T_k\sum_{j=1}^{k-1} T_j
 \end{equation}
 and
 \begin{equation} 
 \label{eq:decsum}
 \sum_{i=1}^n \Big(\prod_{k=i+1}^n (I-T_k)\Big)w_i = \sum_{i=1}^n w_i -\sum_{k=2}^n \Big(\prod_{i=k+1}^n(I-T_i)\Big) T_k \sum_{j=1}^{k-1} w_j
 \end{equation}
\end{lemma}
\begin{proof} By induction. Equality \eqref{eq:decprod} is trivially satisfied for $n=1$. Suppose now that $n\geq2$, and that \eqref{eq:decprod} holds for $n-1$. 
Then
\begin{align*}
 \prod_{i=1}^{n} (I-T_i)&=(I-T_{n}) \prod_{i=1}^{n-1} (I-T_i)\\
 &=(I-T_{n})\bigg(I-\sum_{j=1}^{n-1}T_{j}+ \sum_{k=2}^{n-1} \bigg(\prod_{i=k+1}^{n-1} \left(I- T_i\right)\bigg)T_k\sum_{j=1}^{k-1} T_j\bigg)\\
 &=I-\sum_{j=1}^{n}T_{j} +T_{n}\sum_{j=1}^{n-1}T_{j}+\sum_{k=2}^{n-1}\bigg(\prod_{i=k+1}^{n} \left(I- T_i\right)\bigg)T_k\sum_{j=1}^{k-1} T_j\\
 &=I-\sum_{j=1}^{n}T_{j} +\sum_{k=2}^{n}\bigg(\prod_{i=k+1}^{n} \left(I- T_i\right)\bigg)T_k\sum_{j=1}^{k-1} T_j,
\end{align*}
and the validty of \eqref{eq:decprod} for every $n\in\N^*$ follows by induction.

Equality \eqref{eq:decsum} is trivially satisfied for $n=1$. Suppose now that $n\geq2$ and that \eqref{eq:decprod} holds for $n-1$.  Then
\begin{align*} \sum_{i=1}^{n} \Big(\prod_{k=i+1}^{n} (I-T_k)\Big)w_i &=  \sum_{i=1}^{n-1} \Big(\prod_{k=i+1}^{n} (I-T_k)\Big)w_i +w_{n}\\
&=(I-T_{n})\Big(\sum_{i=1}^{n-1}  w_i -\sum_{k=2}^{n-1} \Big(\prod_{i=k+1}^{n-1} (I-T_i)\Big) T_k\sum_{j=1}^{k-1} w_j \Big) +w_{n}\\
&=  \sum_{i=1}^{n} w_i -T_n\sum_{j=1}^{n-1}w_j-\sum_{k=2}^{n-1} \Big(\prod_{i=k+1}^{n}(I-T_i)\Big) T_k \sum_{j=1}^{k-1} w_j\\
&=  \sum_{i=1}^{n} w_i -\sum_{k=2}^{n} \Big(\prod_{i=k+1}^{n}(I-T_i)\Big) T_k \sum_{j=1}^{k-1} w_j,
\end{align*}
and the conclusion follows.
\end{proof}
The following lemma establishes an equivalent expression for the iterates \eqref{eq:nine1}-\eqref{eq:nine}.
\begin{lemma}\label{lem:rec} Let  $t\in\mathbb{N}$. Then,
\begin{equation}\label{eq:df}
\hat{w}_{t+1} = \prod_{i=1}^n \left(I-\frac{\gamma}{n} T_{x_i}\right) \hat{w}_{t}+\frac{ \gamma}{n}
\sum_{i=1}^n\prod_{k=i+1}^n \left(I-\frac{\gamma}{n} T_{x_k}\right) S^*_{x_{i}}y_i
\end{equation}
\end{lemma}
\begin{proof}
For every $i\in\{1,\dots,n\}$, the update of $\hat{v}^i_t$ in \eqref{eq:nine} can be equivalently written 
as in \eqref{eq:nineop}, and is of the form \eqref{eq:dftv}, with 
$A_r=I- (\gamma/n)T_{x_{r+1}}$, and $B_r=(\gamma/n)S^*_{x_{r+1}}y_{r+1}$, for every 
$r\in\{0,\ldots,n-1\}$, and $X_0=\hat{w}_t$. Equation \eqref{eq:df} follows by writing \eqref{eq:dftvf} for $r=n-1$.
\end{proof}
\begin{proposition}\label{prop:miter} 
For every $t\in\N$, the iteration \eqref{eq:nine1}-\eqref{eq:nine} can be written as
\begin{align}
\label{eq:miter}
&\hat{w}_{t+1}=(I-\gamma\hat{T})\hat{w}_t + {\gamma}\bigg(\frac{1}{n} \sum_{j=1}^n S^*_{x_j}y_j\bigg)+{\gamma^2}\left( \hat{A}\hat{w}_t- \hat{b}\right),
\end{align}
with
\begin{align}\label{eq:Atbt}
&\hat{A}=\frac{1}{n^2}\sum_{k=2}^n \prod_{i=k+1}^n \left(I-\frac{\gamma}{n} T_{x_i} \right)  T_{x_k}\sum_{j=1}^{k-1} T_{x_j},\hspace{0.2cm} \hat{b}=\frac{1}{n^2}\sum_{k=2}^n \prod_{i=k+1}^n \left(I-\frac{\gamma}{n} T_{x_i} \right)  T_{x_k}\sum_{j=1}^{k-1} S^*_{x_j}y_j.&
\end{align}
The iteration \eqref{eq:nine1exp}-\eqref{eq:nineexp} applied to $\EE$ can be expressed as 
\begin{equation}\label{eq:expincr}
{w}_{t+1}=\left(I-{\gamma}T\right){w}_t + {\gamma} S^*g_\rho+{\gamma^2} (Aw_t-b).
\end{equation}
with 
\begin{equation}\label{eq:expab}
A=\frac{1}{n^2}\sum_{k=2}^n \left[\prod_{i=k+1}^n \left(I-\frac{\gamma}{n} T\right)\right]T\sum_{j=1}^{k-1} T
,\hspace{0.2cm}  b=\frac{1}{n^2}\sum_{k=2}^n \left[\prod_{i=k+1}^n \left(I-\frac{\gamma}{n} T\right)\right]T\sum_{j=1}^{k-1}S^*g_\rho.
\end{equation}
\end{proposition}
\begin{proof}
Equations \eqref{eq:miter} and \eqref{eq:Atbt} follow from Lemma~\ref{lem:decprod} and Lemma~\ref{lem:rec} applied with $(\forall i\in\{1,\ldots,n\})$ $T_i=({\gamma}/{n})T_{x_i}$. 
Equations \eqref{eq:expincr} and \eqref{eq:expab} follow from Lemma~\ref{lem:decprod} and Lemma~\ref{lem:rec} applied with $(\forall i\in\{1,\ldots,n\})$ $T_i=({\gamma}/{n})T$. 
\end{proof}

Note that, although not explicitly specified, the operator $A$ and the element $b\in\hh$ in Proposition~\ref{prop:miter} depend on $n$.
Equation~\eqref{eq:miter} allows to compare the update resulting from one epoch of the iteration \eqref{eq:nine1}-\eqref{eq:nine} 
with the one of a standard gradient descent on the empirical error with stepsize $\gamma$, which is given by 
\begin{equation}\label{eq:empgd}
\hat{v}_{t+1}=\left(I-\frac{\gamma}{n} \sum_{j=1}^nT_{x_j}\right) \hat{v}_t+\frac{\gamma}{n} \sum_{j=1}^n S^*_{x_j}y_j
\end{equation}
for an arbitrary $\hat{v}_0\in\hh$. As can be seen  comparing \eqref{eq:miter} and  \eqref{eq:empgd},
In particular, the incremental gradient descent can be interpreted as a perturbed gradient descent step, with perturbation
\[
\hat{e}_t={\gamma^2} \left(\hat{A} \hat{w}_t -\hat{b}\right).
\]
\subsection{Proof of STEP 2}
The following recursive expression is key to get sample bounds estimates, and is at the basis of Lemma~\ref{prop:diffb}.

\begin{lemma}
 \label{prop:diff} In the setting of Section~\ref{sec:setting}, let $t\in\N$ and let $\hat{w}_t$ and $w_t$
be defined as in \eqref{eq:nine1}-\eqref{eq:nine} and $w_t$ be defined as \eqref{eq:nine1exp}-\eqref{eq:nineexp}, respectively.
Define $\hat{A}$ and $\hat{b}$  as in \eqref{eq:Atbt}, and  $A$ and $b$ as in \eqref{eq:expab}.
Then
\begin{equation}
\hat{w}_{t}-w_{t}=\Big(I-\gamma \hat{T}+\gamma^2 \hat{A}\Big)^t (\hat{w}_0-w_0)+\gamma\sum_{k=0}^{t-1}\left(I-\gamma\hat{T}+\gamma\hat{A}\right)^{t-k+1} {\zeta}_{k} 
\end{equation}
with 
\begin{equation}
{\zeta}_{k}=(T-\hat{T}){w}_k +\gamma(\hat{A}-A){w}_k+\bigg( \frac{1}{n}\sum_{i=1}^n\hat{S}^*_{x_i}y_i-S^*g_\rho\bigg)+\gamma(b- \hat{b}).
\end{equation}
\end{lemma}
\begin{proof} We  have
\begin{equation}\label{eq:lem2}
\hat{w}_{t+1}=(I-\gamma \hat{T}+\gamma^2 \hat{A})\hat{w}_t+ \frac{\gamma}{n} \sum_{i=1}^n\hat{S}^*_{x_i}y_i-\gamma^2 \hat{b}\,.
\end{equation}
Adding and subtracting $(-\gamma {\hat{T}}+\gamma^2{\hat{A}})w_t$ it follows from \eqref{eq:expincr} that
\begin{align*}
\hat{w}_{t+1}-w_{t+1}&=(I-\gamma \hat{T}+\gamma^2 \hat{A})(\hat{w}_t-w_t)+\gamma(T-\hat{T})w_t+\gamma^2(\hat{A}-A)w_t\\
&\quad+\gamma\Big( \frac{1}{n}\sum_{i=1}^n\hat{S}^*_{x_i}y_i-S^*g_\rho\Big)+\gamma^2(b- \hat{b})
\end{align*}
Relying on equation \eqref{eq:dftv} we get \eqref{eq:diff1}. 
\end{proof}
\subsection{Proof of STEP 3}
Next we provide  a lemma to bound the norm of the operator appearing in \eqref{eq:diff1},
and acting on the random variable $\zeta_k$ in \eqref{eq:zetak}.
\begin{lemma}\label{ref:norma} In the setting of Section~\ref{sec:setting}, 
let $\gamma\in \, ]0,n\kappa^{-1}]$. Then
\begin{equation}\label{eq:oper}
\|I-\gamma\hat{T}+{\gamma^2} \hat{A}\| \leq 1\,. 
\end{equation}
\end{lemma}
\begin{proof}
It follows from Lemma \ref{lem:decprod}, equation \eqref{eq:decprod} applied with $T_i=T_{x_i}$ and the definition of $\hat{A}_k$ in \eqref{eq:Atbt} that
\begin{equation}
I-\gamma\hat{T}+\gamma^2\hat{A} =\prod_{i=1}^n\left(I-\frac{\gamma}{n} T_{x_i}\right). 
\end{equation}
Since $\nor{T_{x_i}}\leq \kappa$ and by assumption $\gamma/n\leq \kappa^{-1}$,  $\nor{I-(\gamma/n)T_{x_i}}\leq 1$ and the statement follows.
\end{proof}

We next provide a first inequality for the sample error.

\begin{lemma}
 \label{prop:diffb} Let $t\in\N$ and let $\hat{w}_t$ and $w_t$
be defined as in \eqref{eq:nine1}-\eqref{eq:nine} and $w_t$ be defined as \eqref{eq:nine1exp}-\eqref{eq:nineexp}, respectively.
Define $\hat{A}$ and $\hat{b}$  as in \eqref{eq:Atbt}, and  $A$ and $b$ as in \eqref{eq:expab}.
Then
\begin{equation}\label{eq:intb}
\norh{\hat{w}_{t}-w_{t}}\leq  \gamma\big( \|T-\hat{T}\|+ \gamma\|\hat{A}-A\|\big)\sum_{k=0}^{t-1} \|w_k\|_\hh+\gamma t\Big(\big\|\frac{1}{n}\sum_{i=1}^n\hat{S}^*_{x_i}y_i-S^*g_\rho\big\| + \gamma \|b- \hat{b}\| \Big).
\end{equation}
\end{lemma}

\begin{proof}
By Lemma~\ref{prop:diff} we derive that
\begin{equation}\label{eq:diff1}
\hat{w}_{t}-w_{t}=\Big(I-\gamma \hat{T}+\gamma^2 \hat{A}\Big)^t (\hat{w}_0-w_0)+\gamma\sum_{k=0}^{t-1}\left(I-\gamma\hat{T}+\gamma\hat{A}\right)^{t-k+1} {\zeta}_{k} 
\end{equation}
with 
\begin{equation}\label{eq:zetak}
{\zeta}_{k}=(T-\hat{T}){w}_k +\gamma(\hat{A}-A){w}_k+\bigg( \frac{1}{n}\sum_{i=1}^n\hat{S}^*_{x_i}y_i-S^*g_\rho\bigg)+\gamma(b- \hat{b}).
\end{equation}
Since $w_0=\hat{w}_0$, we have
\begin{equation}
\norh{\hat{w}_{t}-w_{t}}\leq \gamma \sum_{k=0}^{t-1} \| I-\gamma \hat{T}+\gamma^2 \hat{A}\|^{k+1} \|\zeta_{k}\| 
\end{equation}
By Lemma~\ref{ref:norma}, for every $j\in\mathbb{N}$, $\| I-\gamma\hat{T}+\gamma^2 \hat{A}\|\leq1$, and therefore 
\begin{equation*}
\norh{\hat{w}_{t}-w_{t}}\leq \gamma (\|T-\hat{T}\|+ \gamma \|\hat{A}-A\|) \sum_{k=0}^{t-1}\|w_k\|+ \Big(\gamma \big\|\frac{1}{n}\sum_{i=1}^n\hat{S}^*_{x_i}y_i-S^*g_\rho\big\| +\gamma^2  \|b- \hat{b}\|\Big) t
\end{equation*}
\end{proof}

\subsection{Proof of STEP 4}
Here, we provide a bound for $\|w_t\|_\hh$. The proof technique is similar to that of known 
results in inverse problems \cite{Eng96}, 
although the bound obtained in \eqref{eq:boundnorm} 
is novel. 
 \begin{lemma}
 \label{lem:ftbound} 
 In the  setting of Section~\ref{sec:setting}, let Assumption~\ref{ass:zero} hold, let  $n\in\mathbb{N}^*$, 
 and let  $\gamma\in\left]0,n\kappa^{-1}\right[$. Let ${w}_0=0$, and 
let $t\in\mathbb{N}$. Then the following hold:
\begin{enumerate} 
\item
\label{lem:ftboundi}
Let Assumption~\ref{ass:uno} hold with  $r\in\left[0,1/2\right]$. Then
\begin{equation}
\label{eq:boundnorm}
\norh{w_{t}} \leq  \max\{\kappa^{r-1/2},\big(\gamma t)^{1/2-r}\}\nor{g}_\rho\,.
\end{equation}
\item 
\label{lem:ftboundii}
Suppose that $\mathcal{O}$ is nonempty. Then there exist $w^\dagger$ as in~\eqref{eq:wdag}
and $\beta\in\RPP$ such that
\begin{equation*}
\norh{w_{t+1}} \leq  \beta.
\end{equation*}
\item
\label{lem:ftboundiii}
Let Assumption~\ref{ass:uno} hold with $r\in\left]1/2,+\infty\right[$.
Then  $w^\dagger$ is well defined and 
\begin{equation}
\label{eq:boundnormh}
\norh{w_{t+1}} \leq  \kappa^{r-1/2}\nor{g}_\rho\,.
\end{equation}
\end{enumerate}
\end{lemma}
\begin{proof}
Let $\epsilon\in\left]0,\kappa\right]$ and set $\eta=\gamma/n$.
By Lemma~\ref{lem:incrvsgrad} and by the spectral theorem 
\cite[equation (2.43)]{Eng96}, we derive
\begin{equation}
\label{eq:hk+1}
w_{t+1}=\sum_{j=0}^{nt-1} S^* \eta\prod_{i={j+1}}^{nt-1} \big(I-\eta\LK\big)g_\rho\,.
\end{equation}
\ref{lem:ftboundi}: By \eqref{eq:hk+1}
\begin{align}
\label{eq:gfgf}\|w_{t+1}&\|_\hh\leq \eta\bigg\|S^*\LK^r \sum_{j=0}^{nt-1}  \big(I-\eta \LK\big)^{nt-j+1}\bigg\|  \nor{g}_{\rho} \\
\nonumber&\leq \sup_{\sigma\in[0,\kappa]}  \sigma^{1/2+r}  \bigg\vert \sum_{j=0}^{nt-1} \eta\big(1-\eta \sigma\big)^{nt-j+1}\bigg\vert \nor{g}_{\rho}\\
\nonumber&=\max\bigg\{\sup_{\sigma\in[0,\epsilon]}  nt \eta\sigma^{1/2+r} , \sup_{\sigma\in[\epsilon,\kappa]} \sigma^{r-1/2}\Big(1-  \big(1-\eta \sigma\big)^{nt}\Big)\bigg\} \nor{g}_{\rho}\\
\nonumber&\leq \psi(\epsilon)\nor{g}_\rho
\end{align}
where $(\forall\epsilon\in\mathbb{R}_+)$ $\psi(\epsilon)=\max\bigg\{ \epsilon^{1/2+r}nt\eta ,  \epsilon^{r-1/2}\bigg\}$. Since $\epsilon$ is arbitrary, 
\begin{equation}
\label{eq:epsi}
\norh{w_{t+1}}\leq \inf_{\epsilon\in\left]0,\kappa\right]} \psi(\epsilon)\nor{g}_\rho\,.
\end{equation}
Now note that $\epsilon\in\left]0,+\infty\right] \mapsto nt\eta\epsilon^{1/2+r}$ is strictly increasing, and has limit equal to zero at zero. On the contrary,  $\epsilon\in\left]0,+\infty\right]\mapsto  \epsilon^{r-1/2}$ is strictly decreasing and $\lim_{\epsilon\to 0^+}\epsilon^{r-1/2}=+\infty$. 
Hence, there exists a unique point $\overline{\epsilon}\in \left]0,+\infty\right]$ such that 
\begin{equation}\label{eq:bareps}
nt\eta\overline{\epsilon}^{\,1/2+r}=\overline{\epsilon}^{\,r-1/2}\qquad\text{and}\qquad \psi(\epsilon)=
\begin{cases}\epsilon^{r-1/2}&\text{if } \epsilon\in\left]0,\overline{\epsilon}\right]\\ 
nt\eta \epsilon^{1/2+r} & \text{if } \epsilon\in\left[\overline{\epsilon},+\infty\right]\,,
\end{cases}
\end{equation}
therefore $\overline{\epsilon}$ is the unique minimizer of $\psi$. 
Solving  for $\overline{\epsilon}$ in \eqref{eq:bareps}, 
we get $\overline{\epsilon}=(nt\eta)^{-1}$.
We derive, again from \eqref{eq:bareps}, that
\[
\min_{\epsilon\in\left]0,\kappa\right]} \psi(\epsilon) =\max\{\kappa^{r-1/2},(\gamma t)^{1/2-r}\}. 
\]
Finally, \eqref{eq:epsi}  yields
\begin{equation*}
\norh{w_{t+1}}\leq \Big(\gamma t\Big)^{1/2-r}\nor{g}_\rho.
\end{equation*}

\ref{lem:ftboundii}: First note that, by Fermat's rule, $S^*g_\rho=Tw^\dag$. 
It follows form Lemma~\ref{lem:incrvsgrad} that
\begin{align*}
w_t-w^\dag&=\bigg(\sum_{j=0}^{nt-1}  \eta T \big(I-\eta T\big)^{nt-j+1}-I\bigg)w^\dagger\\
&= \big(I-\eta T\big)^{nt}w^\dagger.
\end{align*}
Let $(\sigma_m,h_m)_{m\in\NN}$ be an eigensystem of $T$. Since $w^\dagger\in N(T)^\perp$ (see \cite[Proposition 2.3]{Eng96}),
it follows that $w^\dag=\sum_{m\in\NN} \langle w^\dag,h_m\rangle h_m$. Therefore, 
\begin{equation} 
\|w_t-w^\dag\|^2_\hh=\sum_{m\in\NN}  \Big|(1-\eta\sigma_m)^{nt}  \langle w^\dag,h_m\rangle\Big|^2
\end{equation}
Since, for every $m\in\NN$, each summand is bounded by $|\langle w^\dag,h_m\rangle|^2$, 
and $\sum_{m\in\NN} |\langle w^\dag,h_m\rangle|^2$, the Dominated Convergence Theorem yields 
\begin{equation}
\label{eq:aez}
\lim_{t\to+\infty}\|w_t-w^\dag\|^2_\hh=\sum_{m\in\NN}\lim_{t\to+\infty}\Big|(1-\eta\sigma_m)^{nt}   \langle w^\dag,h_m\rangle\Big|^2=0.
\end{equation}
Hence, the sequence $(\|w_t\|_{\hh})_{t\in\NN}$
is bounded. 

\ref{lem:ftboundiii}: Arguing as in the proof of \ref{lem:ftboundi}, it follows from \eqref{eq:hk+1} and Assumption~\ref{ass:uno} that 
\begin{align*}
\|w_{t+1}&\|_\hh\leq \sup_{\sigma\in[0,\kappa]} \sigma^{r-1/2}\Big(1- \big(1-\eta \sigma\big)^{nt}\Big)\bigg\} \nor{g}_{\rho}\\
\nonumber&\leq \kappa^{r-1/2}\nor{g}_\rho
\end{align*}
\end{proof}
\subsection{Proof of STEP 5}
The main novel probabilistic estimates are given  in the following proposition. 
This is the more involved part of the proof, where many
tricks are needed in order to get a manageable expression.
The proof is based on writing the terms $\hat{A}-A$  and $\hat{b}-b$ 
as a sum of martingales and then apply Theorem~\ref{thm:pinelis} to derive 
concentration inequalities. 
We start with the following well-known lemma, 
which is a direct consequence of Theorem~\ref{thm:pinelis} 
(see also \cite{devito05}). 

\begin{lemma}
\label{prop:pin}  In the Setting of Section~\ref{sec:setting}, let Assumption~\ref{ass:zero} hold.
For every $\delta\in\left]0,1\right]$
\begin{equation}
\label{eq:pin1}
\P\left(\Big\|\frac{1}{n}\sum_{i=1}^n T_{x_i}-T\Big\|_{HS}\leq \frac{16\kappa}{3\sqrt{n}}\log\frac 2 \delta\right)\geq 1-\delta,
\end{equation}
and
\begin{equation}
\label{eq:pin2}
\P\left(\Big\|\frac{1}{n}\sum_{i=1}^n S^*_{x_i}y_i-S^*g_\rho\Big\|_{\hh}\leq \frac{16\sqrt{\kappa} M}{3\sqrt{n}}\log\frac 2 \delta\right)\geq 1-\delta.
\end{equation}
\end{lemma}
\begin{proof}
Equation \ref{eq:pin1} follows from Theorem~\ref{thm:pinelis}, since $(T_{x_i}-T)_{1\leq i\leq n}$ is a family of i.i.d. random operators taking
values in the space of Hilbert-Schmidt operators satisfying $\|T\|_{HS}\leq \kappa$ and $\|T_{x_i}\|_{HS}\leq \kappa$ (see also \cite{devito05}).
Equation \ref{eq:pin2} follows from Theorem~\ref{thm:pinelis} applied to the i.i.d. random vectors  $(S^*_{x_i}y_i-S^*f_{\rho})_{1\leq i\leq n}$ in $\hh$
whose norms are bounded by $2\kappa M$.
\end{proof}

\begin{proposition}
\label{lem:ak-ak}  
In the Setting of Section~\ref{sec:setting}, let Assumption~\ref{ass:zero} hold, 
let $\gamma\in\left]0,n\kappa^{-1}\right[$, and let $\delta\in\,]0,1[$. 
Then
\begin{equation}\label{eq:ak-ak}
\P\left(\|\hat{A}-A\|_{HS} \leq \frac{32\kappa^2}{3\sqrt{n}} \log\frac4 \delta \right)\geq 1-\delta,
\end{equation}
and
\begin{equation}\label{eq:bk-bk}
\P\left(\norh{\hat{b}-b}\leq \frac{32\kappa M^2}{3\sqrt{n}} \log\frac4 \delta \right)\geq 1-\delta.
\end{equation}
\end{proposition}
\begin{proof}
We first show a useful decomposition. Recall that
\[
\hat{A}=\frac{1}{n}\sum_{j=2}^n \frac{1}{n}\prod_{i=j+1}^n \left(I-{\gamma} T_{x_i} \right)  T_{x_j}\sum_{l=1}^{j-1} T_{x_l}\qquad
{A}=\frac{1}{n}\sum_{j=2}^n \frac{1}{n}\prod_{i=j+1}^n \left(I-{\gamma} T \right)  T \sum_{l=1}^{j-1} T\,.
\] 
For every $j\in\{2,\ldots,n\}$, set
\[
\hat{B}_{j}=\left[\prod_{i=j+1}^n \left(I-{\gamma} T_{x_i} \right)  \right]T_{x_j}, \qquad {B}_{j}=\left[\prod_{i=j+1}^n \left(I-{\gamma} T \right) \right] T
\]
we have
\begin{align}
\nonumber\hat{A}-A&=\frac{1}{n}\sum_{j=2}^n\hat{B}_{j}\left(\frac{1}{n}\sum_{l=1}^{j-1}T_{x_l}\right)-\frac{1}{n}\sum_{j=2}^nB_{j}\frac{j-1}{n} T\\
\label{eq:Akdec}&=\frac{1}{n}\left[\sum_{j=2}^n \hat B_{j}\left(\frac{1}{n}\sum_{l=1}^{j-1}\left(T_{x_l}- T\right)\right)+Q_n T\right]\,,
\end{align}
with
\begin{equation}\label{eq:qm}
Q_n=\sum_{j=2}^n(\hat B_{j}-B_{j})\,.
\end{equation}
We next bound each term appearing in \eqref{eq:Akdec}.
By Lemma \ref{prop:pin}, with probability greater than $1-\delta$,
\begin{equation}\label{eq:akak1}
\sup_{j\in\{1,\ldots,n\}}\Big\|\frac{1}{n}\sum_{l=1}^{j-1}(T_{x_l}- T)\Big\|_{HS}\leq \frac{16\kappa}{3\sqrt{n}}\log\frac2\delta.
\end{equation}
On the other hand
\begin{equation}\label{eq:akak2}
\|\hat B_{j}\|\leq \prod_{i=j+1}^n\Big\|I-\frac{\gamma}{n} T_{x_i} \Big\|\|T_{x_j}\|\leq \kappa.
\end{equation}
Note that $\sum_{j=2}^n\hat B_{j}\left(1/n\sum_{l=1}^{j-1}(T_{x_l}- T)\right)$ is Hilbert-Schmidt, for $T_{x_l}$ and $T$ are  Hilbert-Schmidt operators, with $\|T_{x_l}\|_{HS}\leq \kappa$ and $\|T\|_{HS}\leq \kappa$, and the family of  Hilbert-Schmidt operators is an ideal with respect to the composition in $L(\mathcal{H})$.
Therefore, by \eqref{eq:akak1} and \eqref{eq:akak2},
\begin{equation}\label{eq:akakfirstterm}
\frac{1}{n}\Big\|\sum_{j=2}^n\hat B_{j}\Big(\frac{1}{n}\sum_{l=1}^{j-1}(T_{x_l}- T)\Big)\Big\|_{HS}\leq 
\frac{16\kappa^2}{3\sqrt{n}}\log\frac2\delta
\end{equation}
holds with probability greater than $1-\delta$,  for any $\delta\in\,]0,1[$.
Next we write the quantity $Q_n$ appearing in the second term in  \eqref{eq:Akdec} as the sum of a martingale.  
For short, we set $\eta=\frac{\gamma}{n}$ 
and  for all $j\in\{2,\ldots, n\}$ we denote
\[
\hat{\Pi}^n_{j}=\prod_{i=j+1}^{n} \left(I-\eta T_{x_i}\right) ,\qquad {\Pi}^n_{j}=\prod_{i=j+1}^{n} \left(I-\eta T\right)\,,
\]
so that from the definition of ${Q}_n$ in \eqref{eq:qm},
\[
{Q}_n=\sum_{j=2}^{n}(\hat\Pi^n_{j}T_{x_j}-\Pi^n_{j} T).
\]
We can derive a recursive update that determine a different expression for the quantity $Q_n$ as follows.
Let $s\in\{1,\ldots,n-1\}$.
\begin{align*}
Q_{s+1}&=\sum_{j=2}^{s+1}(\hat\Pi_{j}^{s+1}T_{x_j}-\Pi_{j}^{s+1}T)\\
&=(T_{x_{s+1}}-T)+\sum_{j=2}^{s}(\hat\Pi_{j}^{s+1}T_{x_j}-\Pi_{j}^{s+1}T)\\
&=(T_{x_{s+1}}-T)+\sum_{j=2}^{s}( (I-\eta T_{x_{s+1}})\hat\Pi_{j}^{s}T_{x_j}-(I-\eta T)\Pi_{j}^{s}T)\\
&=(T_{x_{s+1}}-T)+(I-\eta T_{x_{s+1}})\sum_{j=2}^{s}(\hat\Pi_{j}^{s}T_{x_j}-\Pi_{j}^{s}T)+\eta(T-T_{x_{s+1}})\sum_{j=2}^s\Pi_{j}^{s}T\\
&=(I-\eta T_{x_{s+1}})Q_s+(T_{x_{s+1}}-T)\Big(I-\eta\sum_{j=2}^n\Pi_{j}^{s}T\Big)\,.
\end{align*}
Applying equation \eqref{eq:dftv}, since $Q_1=0$, we get
\begin{align}
\label{eq:martin}&Q_n=\sum_{l=1}^{n}\Theta_{l} 
\end{align}
where, $\Theta_{1}=0$ and, for every $l\in\{2,\ldots,n\}$,
\begin{align}
\nonumber\Theta_{l}&=\prod_{i=l+1}^{n} \left(I-\frac{\gamma}{n}T_{x_i}\right) (T_{x_l}-T)\left(I-\frac{\gamma}{n}\sum_{j=2}^{l-1}\prod_{i=j+1}^{l-1}\left(I-\frac{\gamma}{n}T\right)T \right)\,.
\end{align}
For every $l=1,\ldots,n$
\[
\E[\Theta_{l}]=0,
\]
being $T_{x_2},\ldots, T_{x_n}$ independent and $\E[(T_{x_l}-T)]=0$.  Moreover the conditional expectation
\[
\E[\Theta_{l}\,|\, \Theta_{l+1},\ldots,\Theta_{n}]=0,
\]
since $T_{x_l}$ is independent from $T_{x_{l+1}},\ldots, T_{x_n}$. Therefore the sequence $(\Theta_{l})_{1\leq l\leq n}$  is a martingale difference sequence.
The operator $\Theta_{l}$ is Hilbert-Schmidt, since it is the composition of a Hilbert-Schmidt operator with a continuous one. Moreover, $\|\Theta_{1}\|=0$. Next,  since  the operator $T$ is compact and self-adjoint and $0\leq\gamma/n\leq 1/\|T\|$,  from the spectral mapping theorem, for every $l\in\{2,\ldots,n\}$
\[
\Big\|I-\frac{\gamma}{n}\sum_{j=2}^{l-1}\prod_{i=j+1}^{l-1}\left(I-\frac{\gamma}{n}T\right)T \Big\|=\sup_{x\in[0,1]} \Big\vert1-\sum_{j=2}^{l-1}x \left(1-x\right)^{l-j-1}\Big\vert \,.
\]
We have
\begin{align*}
0&\leq\sum_{j=2}^{l-1}x\left(1-x\right)^{l-j-1}=\sum_{j=2}^{l-1}\left(\left(1-x\right)^{l-j-1}-(1-x)^{l-j}\right)=1-(1-x)^{l-2}\leq 1
\end{align*}
Therefore
\[
\Big\|I-\frac{\gamma}{n}\sum_{j=2}^{l-1}\prod_{i=j+1}^{l-1}\left(I-\frac{\gamma}{n}T\right)T \Big\|\leq 1 \,.
\]
Using the last inequality, we derive
\begin{align*}
 \|\Theta_{l}\|_{HS}&\leq \Big\|\prod_{i=l+1}^n \left(I-\frac{\gamma}{n} T_{x_i} \right)\Big\|\|T_{x_l}-T\|_{HS}\Big\|I-\frac{\gamma}{n}
  \sum_{j=2}^{l-1}\prod_{i=j+1}^{l-1}\left(I-\frac{\gamma}{n}T\right)T \Big\|\\
 &\leq\|T_{x_l}-T\|_{HS} \\
 &\leq 2\kappa\,.
\end{align*}
Then, Theorem \ref{thm:pinelis} applied to $(\Theta_{l})_{i\leq l\leq n}$, yields
\begin{equation}\label{eq:bbk}
\Big\|\frac{1}{n}\sum_{l=1}^n\Theta_{l}\Big\|_{HS}\leq \frac{16\kappa}{3\sqrt{n}}\log\frac2 \delta\,
\end{equation}
 with probability greater than $1-\delta$.
 Therefore, with probability greater than $1-\delta$
\begin{equation}\label{eq:pinmar}
\frac{1}{n}\Big\| Q_nT\Big\|_{HS} \leq \frac{16\kappa^2}{3\sqrt{n}}\log\frac2\delta.
\end{equation}
The statement then follows recalling the decomposition in \eqref{eq:Akdec}, and summing \eqref{eq:pinmar} with \eqref{eq:akakfirstterm}.
From the definition of $\hat{b}$ and $b$  in equations \eqref{eq:miter} and \eqref{eq:expab} respectively, we have
\begin{equation}\label{eq:bk-bkp}
\hat{b}-b=\frac{1}{n^2}\sum_{j=2}^n\hat B_{j}\sum_{l=1}^{j-1}S^*_{x_l}y_l-\frac{1}{n^2}\sum_{j=2}^n B_{j}\sum_{l=1}^{j-1}S^*g_\rho\,,
\end{equation}
and equation \eqref{eq:bk-bk} follows reasoning as in the previous part of the proof.
\end{proof}

\subsection{Sample error}
The proof of the bound on the sample error easily follows from the 
above results. 
\begin{theorem}[Sample error]\label{thm:samp_err}   Let Assumption~\ref{ass:zero} hold. Let  $n\in\mathbb{N}^*$, 
suppose that $\gamma\in\left]0,n\kappa^{-1}\right]$,  and let $\hat{w}_0=w_0=0$. 
Let $\delta\in\left]0,1\right[$, and, for every  $t\in\mathbb{N}^*$, let $\hat{w}_t$ and $w_t$ be defined as in 
 \eqref{eq:nine1}-\eqref{eq:nine} and \eqref{eq:nine1exp}-\eqref{eq:nineexp}, respectively.
Then the following hold:
\begin{enumerate} 
\item
\label{thm:samp_erri}
Let Assumption~\ref{ass:uno} hold, for some $r\in\left[0,1/2\right]$, and let $t\in\NN^*$.
Then, with probability greater than $1-\delta$  
\begin{align}
\label{eq:samplei}
\|\hat{w}_{t}&-w_{t}\|_{\hh} \leq\frac{\log(16/\delta)}{3\sqrt{n}} \left[ (16\sqrt{\kappa}{M}+32\kappa M^2\gamma)\gamma t \right.\\
\nonumber &\left.\quad+ (16\kappa+ 32\kappa^2\gamma)\|g\|_\rho \max\left\{\kappa^{r-1/2} \gamma t,\left(\frac{1-2r}{3-2r}+\frac{2}{3-2r}t^{3/2-r}\right)\gamma^{3/2-r}\right\}\right].
\end{align}
\item
\label{thm:samp_errii}
Let $t\in\NN^*$, and let Assumption~\ref{ass:uno} hold for some $r\in\left[1/2,+\infty\right]$.
Then, with probability greater than $1-\delta$  
\begin{align}
\label{eq:sampleii}
\|\hat{w}_{t}&-w_{t}\|_{\hh} \leq\frac{\log(16/\delta)}{3\sqrt{n}} \left[ 16\sqrt{\kappa}{M}+32\kappa M^2\gamma
+ (16\kappa+ 32\kappa^2\gamma)\|g\|_\rho \kappa^{r-1/2}\right] \gamma t .
\end{align}
\end{enumerate}
\end{theorem}
\begin{proof}
Substituting the bounds obtained in Lemma~\ref{prop:pin} and Proposition~\ref{lem:ak-ak} with $\delta/4$ 
into \eqref{eq:intb}, and applying Lemma~\ref{prop:diffb}, we obtain
 yield that with probability bigger than $1-\delta$
\begin{align*}
\norh{\hat{w}_{t}-w_{t}}\leq \frac{\log(16/\delta)}{3\sqrt{n}}\left(\gamma \big(16\kappa+ 32\kappa^2 \gamma \big)\sum_{k=0}^{t-1} \|w_k\|_\hh+\gamma t\Big(16\sqrt{\kappa}M +32\kappa M^2\gamma \Big)\right).
\end{align*}
Statements \ref{thm:samp_erri} and \ref{thm:samp_errii} directly follow from the bound on $\|w_k\|_\hh$ obtained in Lemma~\ref{lem:ftbound}.
\end{proof}

\subsection{Proof of STEP 6 -- approximation error}
The proof of this result is similar to that of the approximation error bounds obtained in \cite{yao,yiming}, 
and uses spectral techniques, which are classical  in linear inverse problems \cite{Eng96}.

\begin{theorem}[Approximation error]
\label{thm:apprerr} In the setting of Section~\ref{sec:setting}, let Assumption~\ref{ass:zero} hold,
let $n\in\mathbb{N}$, let $w_0\in\hh$, let $\gamma\in\left]0, n\kappa^{-1}\right[$ and let $(w_t)_{t\in\mathbb{N}}$ be 
defined as in \eqref{eq:nine1exp}-\eqref{eq:nineexp}.  
Then the following hold:
\begin{enumerate}
\item
\label{thm:apprerri} The approximation error $\EE(w_t)-\inf_\hh\EE\to 0$.
\item
\label{thm:apprerrii} 
Suppose that $\mathcal{O}$ is nonempty. Then $w^\dagger$ in \eqref{eq:wdag} exists and $\|w_t-w^\dag\|_\hh\to0$.
\item
\label{thm:apprerriii}
Let Assumption~\ref{ass:uno} hold for some $r\in\left]0,+\infty\right[$.
Then 
\[
\EE(w_t)-\inf_\hh \EE\leq  \bigg(\dfrac{r}{\gamma t}\bigg)^{2r} \|g\|_\rho^2\,.
\]
\item
\label{thm:apprerriv}
Let Assumption~\ref{ass:uno} hold for some $r\in\left]1/2,+\infty\right[$.
Then $w^\dag$ in \eqref{eq:wdag} is well-defined and 
\[
\|w_t-w^\dag\|_\hh \leq \bigg(\dfrac{r-1/2}{\gamma t}\bigg)^{r-1/2} \|g\|_\rho\,.
\]
\end{enumerate}
\end{theorem}
\begin{proof}
\ref{thm:apprerri}: This is a direct consequence of Lemma~\ref{lem:grad}.

\ref{thm:apprerrii}: The proof is the same lines as that of Lemma~\ref{lem:ftbound}-\ref{lem:ftboundii}. 
See in particular \eqref{eq:aez}.

\ref{thm:apprerriii}: It follows from Lemma~\ref{lem:incrvsgrad} that
\begin{align*}
\|Sv_{t+1}-g_\rho\|_{\rho}
&=\bigg\|\bigg(\Big(\sum_{j=0}^{nt-1} L\eta \big(I-\eta\LK\big)^{nt-j+1}\Big)-I\bigg) g_\rho\bigg\|_\rho\\
&=\sup_{\sigma\in[0,\|L\|]} \sigma^r\big(1-\eta\sigma\big)^{nt}\|g\|_\rho .
\end{align*}
Note that, the last term is maximized at $\sigma=rn/(\gamma(r+nt+1))$, hence for every $\sigma\in\left[0,+\infty\right[$,
\begin{align*}
\sigma^r \Big(1-\frac{\gamma}{n} \sigma\Big)^{nt} &\leq \Big(\frac{n r}{\gamma(r+nt)}\Big)^r\Big(1-\frac{r}{r+nt}\Big)^{nt}\\
&\leq\Big(\frac{n r}{\gamma(r+nt)}\Big)^r\\
&\leq \Big(\frac{r}{t\gamma}\Big)^r.
\end{align*}
Finally, the equality
\[
\EE(w_t)-\inf_\hh\EE=\|Sw_t- g_\rho\|_\rho^2,
\]
yields the statement.

\ref{thm:apprerriv}: Since $L^{1/2}$ is a partial isometry between $L^2(\hh,\rho_\hh)$ and
$S(\hh)$ and $g_\rho=L^{1/2}\big(L^{r -1/2}g\big)$ by Assumption \eqref{ass:uno}, it follows that $g_\rho\in S(\hh)$, and 
thus $\mathcal{O}$ is nonempty, and $w^\dag$ is well defined. 
Moreover, since $S^*g_\rho=Tw^\dag$, we also get that $Tw^\dag=S^*\LK^rg=T^rS^*g$, 
implying that
$w^\dag=T^\dagger T^{r} S^*g$. It follows from Lemma~\ref{lem:incrvsgrad} and \cite[Equation 2.24]{Eng96} that
\begin{align*}
\norh{w_t-w^\dag}&\leq \Big\| T^\dagger T^{r} \bigg(\sum_{j=0}^{nt-1}  \eta T\big(I-\eta T\big)^{nt-j+1}-I\bigg)S^*\Big\|  \nor{g}_{\rho} \\
&=\sup_{\sigma\in[0,\nor{L}]}  \sigma^{r-1/2}\big(1-\eta \sigma\big)^{nt} \nor{g}_{\rho}\\
&\leq \Big(\frac{r-1/2}{t\gamma}\Big)^{r-1/2}\nor{g}_\rho,
\end{align*}
where the last inequality can be derived proceeding as in \ref{thm:apprerriii}.
\end{proof}

\subsection{STEP 7: proof of the main results}

Let us denote by $\mathcal{S}(t,n,\delta)$ the right hand side of \eqref{eq:samplei}.
Then, combining the sample error estimate with the error decomposition \eqref{eq:dec}, we can immediately
derive the following inequalities
\begin{align*}
\EE(\hat{w}_t)-\inf_\hh\EE&\leq 2\kappa(\mathcal{S}(t,n,\delta))^2+ 2\mathcal{A}(t,\gamma,n)\\
\norh{\hat{w}_t-w^\dag}&\leq \mathcal{S}(t,n,\delta)+ \|w_t-w^\dagger\|
\end{align*}
with probability greater than $1-\delta$.
Note that analogous inequalities hold for the case $r>1/2$ and with respect to the norm in $\hh$.
We are now ready to prove the Theorems stated in Section~\ref{sec:main}.
The proof of universal consistency is a consequence of the sample and approximation 
error bounds, and of the application of Borel-Cantelli Lemma. 

\noindent{\bf Proof of Theorem~\ref{thm:mainnorate}}.
\ref{thm:mainnoratei}: 
Recalling the error decomposition in \eqref{eq:dec}, we have
\[
\EE(\hat{w}_t)-\inf_{\hh}\EE\leq 2\kappa\norh{\hat{w}_t-w_t}^2+2(\EE({w}_t)-\inf_{\hh}\EE).
\]
Suppose that $t>1$. Theorem~\ref{thm:samp_err} applied with $r=0$ yields that there exists 
$c\in\mathbb{R}_{++}$ such that, with probability greater than $1-\delta$  
\begin{equation}
\label{eq:sampin}
\|\hat{w}_t-w_t\|_\hh\leq  c\frac{t^{3/2}\log(16/\delta)}{3\sqrt{n}}+\EE(w_t)-\inf_\hh\EE .
\end{equation}

Since $\sum_{k\in\NN} \gamma=+\infty$ and $\gamma\leq \kappa^{-1}\leq n\kappa^{-1}$,
by Theorem~\ref{thm:apprerr}\ref{thm:apprerri}, $\mathcal{A}(t)=\EE(w_t)-\inf_\hh\EE\to 0$.
Moreover,  $\mathcal{A}(t^*(n))\to 0$ since $t^*(n)\to +\infty$.
Let $\eta\in\RPP$ and let $\alpha\in\left]1,+\infty\right[$. By \eqref{eq:stoprule1},
there exists  $\bar{n}\in\NN$ such that, for every $n\geq \bar{n}$,  $\eta n / t^*(n)^{3}>\alpha \log n$. 
Define
\begin{equation}
A_{n,\eta}=\left\{\EE(\hat{w}_{t^*(n)})-\inf_\hh\EE\geq \mathcal{A}(t^*(n))+ c\eta \right\}.
\end{equation}
By \eqref{eq:sampin}, for every $n\geq \bar{n}$, $\P(A_{n,\eta})\leq 16 \exp(-\eta n/t^*(n)^{3(1-\theta)})\leq \exp(-\alpha\log n)$. 
Therefore, 
\[
\sum_{n\geq \bar{n}} \P(A_{n,\eta}) \leq \sum_{n\geq \bar{n}} n^{-\alpha}<+\infty,
\]
hence the Borel-Cantelli lemma yields $\P(\bigcap_{k\geq \bar{n}}\bigcup_{n\geq k} A_{n,\eta})=0$,
and almost sure convergence follows.

\ref{thm:mainnorateii}:  Since $\mathcal{O}$ is nonempty, it follows that there $g_\rho\in S(\mathcal{H})=L^{1/2}(\hh)$. Therefore,
Assumption~\ref{ass:uno} is satisfied with $r=1/2$. From Theorem~\ref{thm:samp_err}\ref{thm:samp_errii}, that there exists 
$c_1\in\RPP$ such that, with probability greater than $1-\delta$ 
\begin{equation}
\label{eq:sampinh}
\|\hat{w}_t-w_t\|_\hh\leq \frac{c_1 t\log(16/\delta)}{3\sqrt{n}}  
\end{equation}
Moreover, since $\sum_{k\in\NN} \gamma=+\infty$ and $\gamma\leq n\kappa^{-1}$,
by Theorem~\ref{thm:apprerr}\ref{thm:apprerrii}, $\|w_t-w^\dag\|_\hh\to 0$.
Reasoning as in \ref{thm:mainnoratei}, we obtain \eqref{eq:uconsh}.
\endproof

\noindent{\bf Proof of Theorem~\ref{thm:mainH}}. \ref{thm:mainHi}: It follows from 
Theorem~\ref{thm:samp_err}\ref{thm:samp_errii}, that with probability greater than $1-\delta$ 
\begin{equation}
\label{eq:sampinh2}
\|\hat{w}_t-w_t\|_\hh\leq \frac{32\log(16/\delta)}{n}\left[M\kappa^{-1/2}+2M^2\kappa^{-1}+3\kappa^{r-1/2}\|g\|_\rho\right]t
\end{equation}
Moreover, Theorem~\ref{thm:apprerr}\ref{thm:apprerriv} yields
\begin{equation}
\label{eq:errinh}
\|w_t-w^\dag\|_\hh \leq 
\left(\frac{r-1/2}{\gamma t}\right)^{r-1/2}\|g\|_\rho.
\end{equation}
Inequality~\eqref{eq:boundprobh} follows by adding \eqref{eq:sampinh2} with \eqref{eq:errinh}.

\ref{thm:mainHii}:  Let $\alpha\in\left]0,+\infty\right[$ and let $t^*(n)=\lceil n^{\alpha}\rceil$. 
Minimizing the right hand side in \eqref{eq:boundprobh}, we get
\[
\alpha-1/2=\alpha(1/2-r)
\]
leading to the expression of $t^*(n)$. Now, let $n\in\NN^*$ and $\beta\in\left[1,2\right[$ be such that $n^{\alpha}\leq t^*(n)=\beta n^\alpha\leq n^\alpha+1$. Then,
by \eqref{eq:boundprobh} we get 
\begin{equation}
\|\hat{w}_t-w_t\|_\hh\leq \beta^{1/2-r}{2r+1}32\log(16/\delta)\left[M\kappa^{-1/2}+2M^2\kappa^{-1}+3\kappa^{r-1/2}\|g\|_\rho\right]n^{\frac{1/2-r}{2r+1}}
\end{equation}
and
\begin{equation}
\|w_t-w^\dag\|_\hh \leq 
\left(\frac{r-1/2}{\beta }\right)^{r-1/2}\|g\|_\rho n^{\frac{1/2-r}{2r+1}}.
\end{equation}
Equation \eqref{eq:rcorh} follows recalling that $\beta\in\left[1,2\right[$
in \ref{thm:mainrhobi}.
\endproof

\noindent{\bf Proof of Theorem~\ref{thm:mainrhob}}. \ref{thm:mainrhobi}: 
Recalling~\eqref{eq:dec}, it follows from  Theorem~\ref{thm:samp_err}\ref{thm:samp_erri} and
Theorem~\ref{thm:apprerr}\ref{thm:apprerriii}, that
\begin{align}
\label{eq:apprerr1}
\EE(\hat{w}_t)-\inf_{\hh}\EE \leq &2\frac{\big(32\log(16/\delta)\big)^2}{n}\left[M+2M^2\kappa^{-1/2}+3\kappa^{r}\|g\|_\rho\right]^2t^2+\bigg(\dfrac{r}{\gamma t}\bigg)^{2r} \|g\|_\rho^2
\end{align}

\ref{thm:mainrhobii}:  As in the proof of Theorem~\ref{thm:mainH}\ref{thm:mainHii}, set $t=\lceil n^{\alpha}\rceil$. Then, minimizing the right hand side in \eqref{eq:apprerr1},
we derive
\[
2\alpha-1=-2\alpha r
\]  
which gives $\alpha=1/(2r+2)$.
Let $n\in\NN^*$ and $\beta\in\left[1,2\right[$ be such that $n^{\alpha}\leq t^*(n)=\beta n^\alpha\leq n^\alpha+1$. Then, plugging the expression of $t^*(n)$ into \eqref{eq:apprerr1}.
we get
\begin{align*}
\EE(\hat{w}_{t^*(n)})-\inf_{\hh}\EE \leq 8\frac{\big(32\log(16/\delta)\big)^2}{n}\left[M+2M^2\kappa^{-1/2}+3\kappa^{r}\right.&\left.\|g\|_\rho\right]^2 n^{1/(r+1)}\\
&+2\bigg(\dfrac{r}{\gamma}\bigg)^{2r}n^{-r/(r+1)} \|g\|_\rho^2.
\end{align*}
\endproof
\appendix
\section{Non attainable case}
\label{app:nonatt}
\begin{theorem}[Finite sample bounds for the risk -- non attainable case]
\label{thm:mainrho}
In the setting of Section~\ref{sec:setting}, let Assumption~\ref{ass:zero} hold, and let $\gamma\in\left]0,\kappa^{-1}\right]$.
Let Assumption~\ref{ass:uno} be satisfied for some $r\in\left]0,1/2\right]$.
Then the following hold:
\begin{enumerate}
\item
\label{thm:mainrhoi}
For every $t\in\NN^*$, with probability greater than $1-\delta$,  
\begin{align}
\label{eq:boundprobr1}
\nonumber \EE(\hat{w}_t)-\inf_\hh\EE\leq & 8\frac{\Big(32\log(16/\delta)\Big)^2}{n} \left[\left(M+2M^2\kappa^{-1/2}+\frac{6\|g\|_\rho\kappa^{-1/2}}{3-2r}\right) t^{3/2-r}\right.\\
&\hspace{2cm}\left.+3 \kappa^{r}\frac{1-2r}{3-2r}\|g\|_\rho\right]^2+ 2\left(\dfrac{r}{\gamma t}\right)^{2r} \|g\|_\rho^2
\end{align}
\item
\label{thm:mainrhoii}
For the stopping rule $t^*\colon\NN^*\to\NN^*$
\begin{equation}
\label{eq:stop_ruler1}
t^*(n)=\big\lceil n^{\frac{1}{3}}\big\rceil
\end{equation}
with probability greater than $1-\delta$,  
\begin{align}
\label{eq:rcor1}
\nonumber\EE(\hat{w}_{t^*(n)})-\inf_\hh\EE\!\leq\!\left[8\Big(32\log\frac{16}{\delta}\Big)^2 \bigg(M+2M^2\kappa^{-1/2}+\right.&\left.\frac{6\|g\|_\rho}{3-2r}
+\frac{3-6r}{3-2r} \kappa^{r}\|g\|_\rho\right)^2\\
&+\left.2\left(\frac{r}{\gamma}\right)^{2r}\|g\|^2_\rho\right] n^{-2r/3} 
\end{align}
\end{enumerate}
\end{theorem}

As for the attainable case, equation~\eqref{eq:boundprobr1} arises from a form of bias-variance 
(sample-approximation) decomposition of the error. 
Choosing the number of epochs that optimize the bounds in \eqref{eq:boundprobr1},
we derive a priori stopping rules~\eqref{eq:stop_ruler1}
and corresponding bound~\eqref{eq:rcor1}. Again, these  results confirm that the number of epochs 
acts as a regularization parameter and the best choice follows from equation~\eqref{eq:boundprobr1}.  

\noindent{\bf Proof of Theorem~\ref{thm:mainrho}}. \ref{thm:mainrhoi}: This follows from Theorem~\ref{thm:samp_err}\ref{thm:samp_erri} and
Theorem~\ref{thm:apprerr}\ref{thm:apprerriii}.

\ref{thm:mainrhoii}: It follows plugging the expression of $t^*(n)$ into the inequality
in \ref{thm:mainrhoi}.
\endproof


%

\

\begin{thebibliography}{10}

\bibitem{BacDie14}
F.~Bach and A.~Dieuleveut.
\newblock Non-parametric stochastic approximation with large step sizes.
\newblock {\em arXiv:1408.0361}, 2014.

\bibitem{bauer}
F.~Bauer, S.~Pereverzev, and L.~Rosasco.
\newblock On regularization algorithms in learning theory.
\newblock {\em Journal of complexity}, 23(1):52--72, 2007.

\bibitem{Ber97}
D.~P. Bertsekas.
\newblock A new class of incremental gradient methods for least squares
  problems.
\newblock {\em SIAM J. Optim.}, 7(4):913--926, 1997.

\bibitem{BlaKra10}
G.~Blanchard and N.~Kr\"amer.
\newblock Optimal learning rates for kernel conjugate gradient regression.
\newblock In {\em Advances in Neural Inf. Proc. Systems (NIPS)}, pages
  226--234, 2010.

\bibitem{bottou-bousquet-2011}
L.~Bottou and O.~Bousquet.
\newblock The tradeoffs of large scale learning.
\newblock In Suvrit Sra, Sebastian Nowozin, and Stephen~J. Wright, editors,
  {\em Optimization for Machine Learning}, pages 351--368. MIT Press, 2011.

\bibitem{ZhangYu03}
P.~Buhlmann and B.~Yu.
\newblock Boosting with the l2 loss: Regression and classification.
\newblock {\em Journal of the American Statistical Association}, 98:324--339,
  2003.

\bibitem{Caponnetto:2006}
A.~Caponnetto and E.~De~Vito.
\newblock Optimal rates for regularized least-squares algorithm.
\newblock {\em Found. Comput. Math.}, 2006.

\bibitem{CapYao06}
A.~Caponnetto and Yuan Yao.
\newblock Adaptive rates for regularization operators in learning theory.
\newblock {\em Analysis and Applications}, 08, 2010.

\bibitem{Cesa}
N.~Cesa-Bianchi, A.~Conconi, and C.~Gentile.
\newblock On the generalization ability of on-line learning algorithms.
\newblock {\em IEEE Transactions on Information Theory}, 50(9):2050--2057,
  2004.

\bibitem{Cesal}
N.~Cesa-Bianchi and G.~Lugosi.
\newblock {\em Prediction, learning, and games}.
\newblock Cambridge University Press, 2006.

\bibitem{CucZho07}
F.~Cucker and D.~X. Zhou.
\newblock {\em Learning Theory: An Approximation Theory Viewpoint}.
\newblock Cambridge University Press, 2007.

\bibitem{devito05}
E.~De~Vito, L.~Rosasco, A.~Caponnetto, U.~De~Giovannini, and F.~Odone.
\newblock Learning from examples as an inverse problem.
\newblock {\em Journal of Machine Learning Research}, 6:883--904, 2005.

\bibitem{devros04}
E.~De~Vito, L.~Rosasco, A.~Caponnetto, M.~Piana, and A.~Verri.
\newblock Some properties of regularized kernel methods.
\newblock {\em Journal of Machine Learning Research}, 5:1363--1390, 2004.

\bibitem{Eng96}
H.~W. Engl, M.~Hanke, and A.~Neubauer.
\newblock {\em Regularization of inverse problems}.
\newblock Kluwer, 1996.

\bibitem{HuaAvr14}
P.-S. Huang, H.~Avron, T.~Sainath, V.~Sindhwani, and B.~Ramabhadran.
\newblock Kernel methods match deep neural networks on timit.
\newblock In {\em IEEE ICASSP}, 2014.

\bibitem{lecun-98b}
Y.~LeCun, L.~Bottou, G.~Orr, and K.~Muller.
\newblock Efficient backprop.
\newblock In G.~Orr and Muller K., editors, {\em Neural Networks: Tricks of the
  trade}. Springer, 1998.

\bibitem{nedic2001incremental}
A.~Nedic and D.~P Bertsekas.
\newblock Incremental subgradient methods for nondifferentiable optimization.
\newblock {\em SIAM Journal on Optimization}, 12(1):109--138, 2001.

\bibitem{Nem09}
A.~Nemirovski, A.~Juditsky, G.~Lan, and A.~Shapiro.
\newblock Robust stochastic approximation approach to stochastic programming.
\newblock {\em SIAM J. Optim.}, 19(4):1574--1609, 2008.

\bibitem{Nem86}
A.~Nemirovskii.
\newblock The regularization properties of adjoint gradient method in ill-posed
  problems.
\newblock {\em USSR Computational Mathematics and Mathematical Physics},
  26(2):7--16, 1986.

\bibitem{Pistol}
F.~Orabona.
\newblock Simultaneous model selection and optimization through parameter-free
  stochastic learning.
\newblock {\em NIPS Proceedings}, 2014.

\bibitem{Pin94}
I.~Pinelis.
\newblock Optimum bounds for the distributions of martingales in {B}anach
  spaces.
\newblock {\em Ann. Probab.}, 22(4):1679--1706, 1994.

\bibitem{Pol87}
B.~Polyak.
\newblock {\em Introduction to Optimization}.
\newblock Optimization Software, New York, 1987.

\bibitem{RamSil05}
J.~Ramsay and B.~Silverman.
\newblock {\em Functional Data Analysis}.
\newblock Springer Series in Statistics. Springer-Verlag, New York, 2005.

\bibitem{raskutti}
G.~Raskutti, M.~Wainwright, and B.~Yu.
\newblock Early stopping for non-parametric regression: An optimal
  data-dependent stopping rule.
\newblock In {\em in 49th Annual Allerton Conference}, pages 1318--1325. IEEE,
  2011.

\bibitem{SmaZho05}
S.~Smale and D.~Zhou.
\newblock {Shannon sampling II: Connections to learning theory}.
\newblock {\em Applied and Computational Harmonic Analysis}, 19(3):285--302,
  November 2005.

\bibitem{SmaZho07}
S.~Smale and D.-X. Zhou.
\newblock Learning theory estimates via integral operators and their
  approximations.
\newblock {\em Constr. Approx.}, 26(2):153--172, 2007.

\bibitem{SreSriTew12}
N.~Srebro, K.~Sridharan, and A.~Tewari.
\newblock Optimistic rates for learning with a smooth loss.
\newblock {\em arXiv:1009.3896}, 2012.

\bibitem{SteiChri08}
I.~Steinwart and A.~Christmann.
\newblock {\em Support Vector Machines}.
\newblock Springer, 2008.

\bibitem{SteinwartHS09}
I.~Steinwart, D.~R. Hush, and C.~Scovel.
\newblock Optimal rates for regularized least squares regression.
\newblock In {\em COLT}, 2009.

\bibitem{TarYao14}
P.~Tarr{\`e}s and Y.~Yao.
\newblock Online learning as stochastic approximation of regularization paths:
  optimality and almost-sure convergence.
\newblock {\em IEEE Trans. Inform. Theory}, 60(9):5716--5735, 2014.

\bibitem{vapnik1998statistical}
V.N. Vapnik.
\newblock {\em {Statistical learning theory}}, volume~2.
\newblock Wiley, New York, 1998.

\bibitem{yao}
Y.~Yao, L.~Rosasco, and A.~Caponnetto.
\newblock On early stopping in gradient descent learning.
\newblock {\em Constructive Approximation}, 26(2):289--315, 2007.

\bibitem{yiming}
Y.~Ying and M.~Pontil.
\newblock Online gradient descent learning algorithms.
\newblock {\em Foundations of Computational Mathematics}, 8(5):561--596, 2008.

\end{thebibliography}
\end{document}